\documentclass[english]{elsarticle}
\usepackage{geometry}
\usepackage{float}
\usepackage{textcomp}
\usepackage{amsmath}
\usepackage{amsthm}
\usepackage{amssymb}
\usepackage{graphicx}
\usepackage{multirow}
\usepackage{rotating}
\PassOptionsToPackage{normalem}{ulem}
\usepackage{ulem}
\usepackage{algorithm}

\usepackage{algpseudocode}

\makeatletter

\DeclareFontEncoding{LGR}{}{}

\ProvideTextCommand{\~}{LGR}[1]{\char126#1}

\theoremstyle{plain}
\newtheorem{thm}{\protect\theoremname}
\ifx\proof\undefined
\newenvironment{proof}[1][\protect\proofname]{\par
\normalfont\topsep6\p@\@plus6\p@\relax
\trivlist
\itemindent\parindent
\item[\hskip\labelsep\scshape #1]\ignorespaces
}{%
\endtrivlist\@endpefalse
}
\providecommand{\proofname}{Proof}
\fi
\theoremstyle{plain}

\newtheorem{definition}[thm]{\protect\definitionname}

\makeatother

\usepackage{babel}
\providecommand{\lemmaname}{Lemma}
\providecommand{\theoremname}{Theorem}
\providecommand{\definitionname}{Definition}
\providecommand{\examplename}{Example}

\begin{document}

\begin{frontmatter}{}

\title{Random Normed k-Means: A Paradigm-Shift in Clustering within Probabilistic Metric Spaces}

\author[ens]{Abderrafik Laakel Hemdanou\corref{cor1}}
\ead{abderrafik.laakelhemdanou@etu.uae.ac.ma}

\author[ens]{Youssef Achtoun}
\ead{achtoun44@outlook.fr}

\author[ens]{Mohammed Lamarti Sefian}
\ead{lamarti.mohammed.sefian@uae.ac.ma}

\author[ens]{Ismail Tahiri}
\ead{istahiri@uae.ac.ma}

\author[ensias]{Abdellatif El Afia}
\ead{abdellatif.elafia@ensias.um5.ac.ma}

\cortext[cor1]{Corresponding author.}

\address[ens]{
 Applied Mathematics and computer science team, Higher normal school, Martil, Morocco}
\address[ensias]{
 National Higher School for Computer Science and Systems Analysis, Mohammed V University in Rabat, Rabat, Morocco}

\begin{abstract}
Existing approaches remain largely constrained by traditional distance metrics, limiting their effectiveness in handling random data. In this work, we introduce the first k-means variant in the literature that operates within a probabilistic metric space, replacing conventional distance measures with a well-defined distance distribution function. This pioneering approach enables more flexible and robust clustering in both deterministic and random datasets, establishing a new foundation for clustering in stochastic environments. By adopting a probabilistic perspective, our method not only introduces a fresh paradigm but also establishes a rigorous theoretical framework that is expected to serve as a key reference for future clustering research involving random data. Extensive experiments on diverse real and synthetic datasets assess our model’s effectiveness using widely recognized evaluation metrics, including Silhouette, Davies-Bouldin, Calinski-Harabasz, the adjusted Rand index, and distortion. Comparative analyses against established methods such as k-means++, fuzzy c-means, and kernel probabilistic k-means demonstrate the superior performance of our proposed random normed k-means (RNKM) algorithm. Notably, RNKM exhibits a remarkable ability to identify nonlinearly separable structures, making it highly effective in complex clustering scenarios. These findings position RNKM as a groundbreaking advancement in clustering research, offering a powerful alternative to traditional techniques while addressing a long-standing gap in the literature. By bridging probabilistic metrics with clustering, this study provides a foundational reference for future developments and opens new avenues for advanced data analysis in dynamic, data-driven applications.

\end{abstract}

\end{frontmatter}{}

\section{Introduction}
K-means stands out as one of the most widely employed algorithms in unsupervised machine learning, celebrated for its simplicity and effectiveness in clustering data points into cohesive groups \cite{b1,liu2011clustering}. As its popularity persists, numerous iterations of K-means have emerged, each striving to augment its performance across diverse datasets and applications \cite{b2, b3, b4}. This pursuit of optimizing K-means clustering has led to a rich tapestry of research journey, exploring dimensions such as dimensionality reduction \cite{b5,LaakelHemdanou2024ComparativeAO}, computational efficiency \cite{b6} and robustness to varying data distributions \cite{b7}. At the core of these efforts lies the investigation of metric spaces, where distances between data points dictate clustering decisions \cite{b8}. Traditional K-means algorithms operate within deterministic metric spaces, where distances are uniquely defined by fixed mathematical formulations \cite{b9}. However, real-world datasets often present challenges beyond deterministic metrics, with inherent complexities such as nonlinear structures and noisy observations \cite{b10}. In response, there is a call for a paradigm reorientation, one that embraces the inherent uncertainty and variability present in real-world data.
\\
\vskip 1mm

Motivated by this imperative, we introduce a novel variant of K-means clustering rooted in the framework of probabilistic metric spaces. Unlike its predecessors, our approach eschews the rigidity of deterministic metrics, opting instead for a flexible distance distribution function that encapsulates the probabilistic nature of distances between data points \cite{b11}. From the theory of probabilistic metric spaces, it is possible to construct a metric as a result of joining through an appropriate t-norm with metrics used in clustering methods. For example, these metrics are being used in different fields of science. Karakli{\'c} et al. \cite{b12} introduce an image filtering method by grouping a family of computationally efficient filters with strong detail preservation capabilities. They achieve this by combining probabilistic metrics, utilizing the UIQI (Universal Image Quality Index) quality index. Torra \cite{b13, b14, b15} proposes probabilistic metrics designed to quantify the distance between models generated from data. Additionally, Sato-Ilic presents a similar approach, employing kernel-based metrics for multidimensional scaling \cite{b16}. Finally, Pourmoslemi et al. \cite{Pour} propose a D2D multiple-metric routing protocol aimed at enhancing network stability and path efficiency in the aftermath of potential natural disasters, utilizing probabilistic normed spaces.\\
\vskip 1mm

By adopting a probabilistic lens, we transcend the confines of deterministic clustering, offering a versatile framework capable of accommodating both deterministic and random databases.
\vskip 1mm

In this paper, we embark on a journey to elucidate the intricacies of our proposed probabilistic metric space approach to K-means clustering \cite{b17}. We will compare our proposed method with three established K-means variants: K-means++, kernel probabilistic K-means and fuzzy C-means. Through a series of empirical evaluations across diverse datasets, we scrutinize the performance of our model, comparing it against these benchmarks and state of the art methodologies. Leveraging a battery of evaluation metrics including Silhouette \cite{lenssen2024medoid}, Davies Bouldin \cite{Ros2023PDBIAP}, Calinski-Harabasz \cite{Hu2024ComparisonOK}, the adjusted rand index \cite {Warrens2022UnderstandingTA} and distortion \cite{b1}. We seek to unravel the efficacy and robustness of our approach across varied clustering scenarios \cite{b9, b18}. Our findings not only underscore the potency of the probabilistic metric space paradigm in reshaping the landscape of data clustering but also herald new frontiers for exploration and innovation in unsupervised learning \cite{b19}. By bridging the gap between theory and practice, we aspire to empower practitioners with a versatile toolkit capable of unraveling the latent structures embedded within complex datasets \cite{b20}.\\

The paper is structured as follows. In Section 2, we review two closely related works. Section 3 introduces the proposed metric, which serves as the foundation for our clustering method based on probabilistic metric spaces, particularly emphasizing the adoption of random normed space. In Section 4, we present two methods: the first one, named spectral sampling for RNKM, focuses on centroid estimation, while the second one elaborates on our novel clustering approach. Additionally, we delve into the time complexity of our method and compare it with other approaches. Results and discussions are outlined in Section 4. Finally, Section 5 draws conclusions based on the findings presented in the preceding sections.
\vskip 1mm

\section{Related work}
Numerous clustering algorithms are documented in the literature, each designed to uncover patterns and achieve the most coherent representations from the provided data. In addition to discussing the classical methods, namely the K-means and K-means++ algorithms \cite{ikotun2023k,cohen2022towards}, this section will also provide an overview of two other methodologies closely aligned with our approach.

\subsection{K-Means and K-Means++}
K-means stands as a prominent unsupervised learning technique in data analysis and machine learning \cite{b26, b25, sinaga2020unsupervised}. Its primary goal lies in clustering datasets into distinct groups (clusters) where elements within each group exhibit higher similarity to one another compared to elements in different groups. The formation of clusters is achieved by computing the distance between input data points and cluster centroids. In the K-Means algorithm, multiple centroid centers are computed within the input data, all sharing the same k-value \cite{isaac2021fiber,saravanan2023modified}. This process facilitates the grouping of data points with values proximate to the centroid into a cluster. The Euclidean distance serves as the fundamental metric for measuring the distance between individual data points and cluster centroids \cite{bundak2022fuzzy,hamalainen2020improving}. When the distance between input data and a centroid yields the smallest value compared to other clusters\cite{hakkoymaz2009applying}, the data is assigned to that cluster.
The classic k-means objective function aims to minimize the within-cluster sum of squares (WCSS). This objective function can be expressed as follows:

\[
\mathcal{J}(C) = \sum_{i=1}^{l} \sum_{j=1}^{r(A)} \| c_i- x_j  \|^2,
\]

where:
\begin{itemize}
    \item $l$ is the number of clusters.
    \item $A$ is the set of points.
    \item $r(A)$ is the number of points.
    \item $c_i$ is the centroid of the $i$-th cluster, computed as $c_i=\frac{1}{n_i}\sum_{j=1}^{r(A)}x_j$ with $n_i$ is the number of data points assigned to cluster $i$.
\end{itemize}

The goal of the k-means algorithm is to find the set of centroids $C = \{c_1, c_2, ..., c_l\}$ that minimizes the object
\[{{\arg  min\ }}_j{\sum_{i=1}^{l} \sum_{j=1}^{r(A)} \| c_i- x_j  \|^2}.\] 

For k-means++ \cite{b3}, the optimization problem remains the same as in classical k-means. However, the pivotal disparity lies in the initialization step. In classical k-means, initial centroids are randomly chosen from the dataset, whereas in k-means++, the initialization process is more refined. It selects centroids based on their distances from existing centroids. This approach ensures a more even spread of initial centroids, often resulting in quicker convergence and enhanced clustering quality compared to classical k-means.

\subsection{Fuzzy C-Means}
Fuzzy C-means (FCM) stands out as a notable unsupervised learning approach widely utilized in data analysis and machine learning tasks. Similar to k-means, its fundamental objective centers on partitioning datasets into cohesive groups or clusters where elements within each cluster demonstrate higher similarity to one another compared to elements in different clusters. The clustering process in FCM involves iteratively assigning data points to clusters based on their degrees of membership, which reflect the level of similarity between data points and cluster centroids. Unlike the crisp assignments in k-means, FCM allows for soft assignments, enabling data points to belong to multiple clusters simultaneously with varying degrees of membership. The computation of cluster centroids in FCM is guided by the weighted average of all data points, with weights determined by the degree of membership. The optimization in FCM aims to minimize the weighted sum of squared errors \cite{vadaparthi2011segmentation,anand2013semi,el2023fp}, which quantifies the discrepancy between data points and cluster centroids, considering their membership degrees. This optimization process iteratively updates cluster centroids and membership degrees until convergence is achieved. The objective function for FCM can be expressed as follows:

$$ \mathcal{J}(U,C)=\sum_{i=1}^{r(A)}\sum_{j=1}^{l}u_{i,j}^{m}\left\|x_{i}-c_{j}\right\|^{2},$$
where:
\begin{itemize}
\item $u_{ij}$ represents the degree of membership of data point $i$ to cluster $j$, computed as 
$$ u_{i,j}=\left(\sum_{k=1}^{l}\left(\frac{\left\|x_i-c_j\right\|}{\left\|x_i-c_k\right\|}\right)^{\frac{2}{m-1}}\right)^{-1}.$$
\item $m$ is a weighting exponent controlling the degree of fuzziness (usually $m >1$, typically set to 2).
\item $x_i$ is the $i$-th data point.
\item $c_j$ is the centroid of cluster $j$, calculate as: 
$$ c_j=\frac{\sum_{i=1}^{n} u_{i,j}^{m}x_i}{\sum_{i=1}^{n}u_{i,j}^{m}}.$$
\end{itemize}

The objective of FCM is to find the optimal cluster centroids $C= {c_1, c_2, ..., c_l}$ and membership degrees $U$ that minimize the objective function $J(U,C)$. This optimization process involves iteratively updating the centroids and membership degrees until convergence is achieved.

\subsection{Kernel Probabilistic K-Means}
Kernel probabilistic k-means (KPKM) \cite{b17} emerges as a powerful extension of traditional k-means, offering enhanced flexibility and robustness in clustering tasks within the realm of machine learning and data analysis. Much like its predecessor, KPKM aims to partition datasets into cohesive clusters, but with a probabilistic framework that accommodates uncertainty inherent in real world data. The essence of KPKM lies in its utilization of kernel functions, enabling nonlinear transformations that can capture intricate structures within the data. By incorporating probabilistic principles, KPKM assigns data points to clusters not only based on their spatial proximity to cluster centroids but also considering the probability of belonging to each cluster. This probabilistic assignment allows for a more nuanced representation of data points, particularly beneficial when dealing with overlapping clusters or noisy data \cite{de2023k}. The optimization objective in KPKM involves finding the optimal cluster centroids and covariance matrices that maximize the likelihood of observing the data given the model parameters. This iterative optimization process iteratively refines the cluster representations until convergence\cite{klopotek2022k}, yielding robust and interpretable clustering solutions tailored to the complexities of the underlying data distribution. The objective function for KPKM is given as: 

$$ \mathcal{J}(U,C)=\sum_{i=1}^{r(A)}\sum_{j=1}^{l}u_{i,j}\left\|\phi(x_{i})-c_{j}\right\|^{2},$$
where : 
\begin{itemize}
	\item $u_{i,j}$ is the membership weight of data point $x_i$ in cluster $j$. It is typically calculated using the Gaussian Radial Basis Function (RBF) kernel and the softmax function to ensure that the memberships sum up to 1 for each data point by 
	$$ u_{i,j} = \frac{e^{-\frac{\|\phi(x_i)-c_j\|^2}{2\sigma^2}}}{\sum_{k=1}^{l}e^{-\frac{\|\phi(x_i)-c_k\|^2}{2\sigma^2}}},$$
	with $\sigma$ is the bandwidth parameter of the RBF kernel.
	\item $\phi(x_i)$ is the feature vector of data point $x_i$ mapped into a high-dimensional space using a chosen kernel function. In the case of the RBF kernel, $\phi(x_i)$ is often defined as the result of applying the kernel function to $x_i$, 
	$$\phi(x_i) = \text{RBF}(x_i) = \exp\left(-\frac{\|x_i - \mu\|^2}{2\sigma^2}\right),$$
with $\mu$ is the mean or center of the kernel function.
\item  $c_j$ is the centroid of cluster $j$, computed as:
$$ c_{j}=\frac{\sum_{i=1}^{r(A)}u_{i,j}\phi(x_{i})}{\sum_{j=1}^{r(A)}u_{i,j}}.$$ 
\end{itemize}

\section{Random Normed spaces and Its Construction}
In the subsequent sections, we delve deeper into the conceptual underpinnings of probabilistic metric spaces, elucidate the methodology underpinning our proposed approach, and present a comprehensive analysis of experimental results. We try to illuminate the path towards a more nuanced and adaptive framework for data clustering that embraces uncertainty \cite{hossain2019dynamic}, fosters innovation and unlocks the latent potential inherent in unsupervised learning paradigms. Our approach diverges from traditional clustering methods, opting instead for a k-means clustering strategy within the framework of probabilistic metric spaces. The origin of this concept can be traced back to K. Menger's work in 1942 \cite{b21}, where distribution functions were introduced as substitutes for non negative real numbers in metric values. This innovative approach accommodates scenarios where the precise distances between two points remain unknown \cite{deng2023query,gonzales2022distance}, with only the probabilities associated with potential distance values available.

To commence our exploration, it is crucial to define the basic concepts underlying probabilistic normed and metric spaces.
\subsection{Probabilistic metric spaces}
\begin{definition}
Let $\Lambda^{+}$ be the space of all distance distribution functions $\zeta:[0,+\infty] \rightarrow [0,1]$  such that:
\begin{enumerate}
	\item $\zeta$ is left continuous on $[0,+\infty)$,
	\item $\zeta$ is non-decreasing,
	\item $\zeta(0)=0$ and  $\zeta(+\infty)=1$.
\end{enumerate}
We indicate by $\triangle^{+} \subset \Lambda^{+}$ the set $ \triangle^{+}=\left\{ \zeta \in \Lambda^{+} : \lim\limits_{x\to +\infty} \zeta(x) = 1 \right\} $.
\end{definition}
One of the element of $\triangle^{+}$ is $\epsilon_{a}$ which defined as
 $$\epsilon_{a}(x)=\left\{
\begin{array}
[c]{lll}%
0 & \ \ \ \ \ \mbox{if} &   x \in [0,a],\\
1 & \ \ \ \ \ \mbox{if} & x \in (a,+\infty].
\end{array}
\right.  $$

We should note that for each $\zeta,\zeta' \in \Lambda^{+}$ we have $ \zeta \leq \zeta'$ whenever $\zeta(t) \leq \zeta'(t)$ for all $t \geq 0$.

\begin{definition} \cite{b22}
 A binary operation $\pi$ on $\Lambda^{+}$ is called a triangle function if for all $\zeta,\zeta',\zeta" \in \Lambda^{+}$ the conditions listed below are verified:
\begin{enumerate}
	\item $\pi(\zeta,\zeta')=\pi(\zeta',\zeta)$,
	\item $ \pi(\zeta,\pi(\zeta',\zeta"))=\pi(\pi(\zeta,\zeta'),\zeta")$,
	\item  $ \zeta\leq \zeta" \Longrightarrow \pi(\zeta,\zeta') \leq \pi(\zeta",\zeta') $,
	\item  $ \pi(\zeta,\epsilon_{0}) = \pi(\epsilon_{0},\zeta) = \zeta $.
\end{enumerate}
\end{definition}
A mapping $\mathfrak{H}: \left[0,1\right] \times \left[0,1\right] \rightarrow \left[0,1\right]$ is referred to as a t-norm if and only if it has the properties of commutativity, associativity, non-decreasing at each location and has 1 as neutral element.\\
Some of the fundamental t-norms are
 $$\mathfrak{H}_{M}(x,y) = Min(x,y),\ \ \ \mathfrak{H}_{P}(x,y) = x \cdot y \ \ \ and \ \ \mathfrak{H}_{L}(x,y) = Max ( x+y-1, 0 ).$$
A special class of triangle functions will be presented by assuming the form: for any $f,g \in \Lambda^{+}$ and for all $ t \in [0,\infty)$,
$$ \pi_{\mathfrak{H}}(f,g)(t)=\sup\left\{\mathfrak{H}\left(f(s),g(u)\right): \ s+u=t\right\}, $$
with $\mathfrak{H}$ is a left-continuous t-norm. Additionally, if $\mathfrak{H}$ is continuous, $\pi_{\mathfrak{H}}$ must likewise be continuous.

	\begin{definition}\label{def2}
 The triplet $(\Omega,\digamma,\pi)$ is told a probabilistic metric space where $\Omega$ is nonempty set, $\digamma$ is a function from $ \Omega \times \Omega$ into $ \Lambda^{+}$ and $\pi$ is a triangle function if for all $ p,q,r \in \Omega$ and $x,y> 0$, the conditions listed below are verified:
\begin{enumerate}
	\item $ \digamma_{p,q}= \epsilon_{0}$ if and only if $p = q $, \label{item:second2}
	\item $ \digamma_{p,q}=\digamma_{q,p}$,                              \label{item:third2}
	\item $ \digamma_{p,q}\geq \pi(\digamma_{p,r},\digamma_{r,q})$.  \label{item:fourth2}
\end{enumerate}
When $\pi = \pi_{\mathfrak{H}}$ for a given t-norm $\mathfrak{H}$, then the space $(\Omega,\digamma,\pi_{\mathfrak{H}})$ is referred to as Menger space and we write $(\Omega,\digamma,\mathfrak{H})$ which implies that the condition \ref{item:fourth2} equivalent to
\begin{equation}
\label{eq2}
\digamma_{p,r}(x+y) \geq \mathfrak{H}(\digamma_{p,q}(x),\digamma_{q,r}(y)).
\end{equation}
\end{definition}

\subsection{Random normed spaces}

Expanding on Menger's groundwork, Šerstnev extended the concept to probabilistic normed spaces \cite{b23}, providing a comprehensive framework that considers the uncertainty inherent in distance measurements. For a detailed exploration of the features of probabilistic normed spaces, readers can refer to \cite{b22}. In our study, we leverage the principles of probabilistic normed spaces to establish a foundation for employing k-means clustering within this novel context.

\begin{definition}
\label{def1}
Let $\Omega$ be a real vector space, $\Gamma$ is a function from $\Omega$ into $ \Lambda^{+}$, $\pi$ is a continuous triangle function and the following conditions are satisfied for all $ p,q \in \Omega$, $x >0$ and $\alpha \neq 0$,
\begin{enumerate}
	\item $ \Gamma_{p}= \epsilon_{0}$ if and only if   $p = 0$ (the null vector), \label{item:first}
	\item $ \Gamma_{\alpha p}(x)=\Gamma_{p}(\frac{x}{\left|\alpha \right|})$, \label{item:second}
	\item $ \Gamma_{p+q}\geq \pi(\Gamma_{p,r},\Gamma_{r,q}) $. \label{item:third}
\end{enumerate}
If $\pi = \pi_{\mathfrak{H}}$ for some t-norm $\mathfrak{H}$, then the triple $(\Omega,\Gamma,\pi_{\mathfrak{H}})$ is called a random normed space, which implies that the condition \ref{item:third} become 
\begin{equation}\label{eq1}
\Gamma_{p+q}(x+y) \geq \mathfrak{H}(\Gamma_{p}(x),\Gamma_{q}(y)) \ \ for \ all \ x,y >0.
\end{equation}
\end{definition}

\begin{thm} \label{exem1}
Let $(\Omega,\left\|\cdot \right\|)$ be a real normed space and define $\Gamma: \Omega \times \Omega \rightarrow \Lambda^{+}$ by 
$$\Gamma_{p}(t)=\left\{
\begin{array}
[c]{lllll}%
\frac{t}{t+\left\|p\right\|} & \ \ \ \ \ \mbox{if} &   t>0,\\
0 & \ \ \ \ \ \mbox{if} & t=0.\end{array} \right. $$
Then, $(\Omega,\Gamma,\mathfrak{H}_{P})$ is a random normed space. 
\end{thm}
\begin{proof}
In fact, we have to show the probabilistic triangle inequality, since that \ref{item:first} and \ref{item:second} of definition \ref{def1} are trivially satisfied.\\
\begin{align*}
\pi_{\mathfrak{H}_{P}}(\Gamma_{p}(t),\Gamma_{q}(t')) &=\frac{t}{t+\left\|p\right\|} \cdot \frac{t'}{t'+\left\|q \right\|} \\
&= \frac{1}{1+\frac{\left\|p\right\|}{t}} \cdot \frac{1}{1+\frac{\left\|q\right\|}{t'}} \\
& \leq \frac{1}{1+\frac{\left\|p\right\|}{t+t'}} \cdot \frac{1}{1+\frac{\left\|q\right\|}{t+t'}} \\
& \leq \frac{1}{1+\frac{\left\|p \right\|+ \left\|q \right\|}{t+t'}}\\
& \leq \frac{1}{1+\frac{\left\|p+q\right\|}{t+t'}} \\
& = \frac{t+t'}{t+t'+\left\|p+q\right\|} \\
& = \Gamma_{p+q}(t+t').
\end{align*}
Therefore, $(\Omega,\Gamma,\mathfrak{H}_{P})$ is a random normed space. Also, we can show easily that $(\Omega,\Gamma,\pi_{\mathfrak{H}_{M}})$ is random normed space.
\end{proof}
 
The function $\Gamma_{p-q}(t)$ can be understood as the measure of closeness between $p$ and $q$ concerning the parameter $t$, referred to as the sensitivity of the random normed space. When a random normed function $\Gamma$ on $\Omega$ remains unchanged with respect to $t$, it is termed stationary. In other words, for every $p, q \in \Omega$, the function $\Gamma_{p-q}(t) = \Gamma_{p,q}(t)$ remains constant. Such a metric suffices for establishing a topology $\mathcal{N}$ using the basis of neighbors defined by
$$ \mathcal{N}= \left\{ N_{p}(\epsilon,\lambda) : \ p \in X, \ \epsilon > 0 \ \ and \ \ \lambda > 0 \right\}, $$
where
$$ \mathcal{N}_{p}(\epsilon,\lambda) = \left\{ q \in X : \ \Gamma_{p-q}(\epsilon) > 1-\lambda \right\}. $$

This topology enables the establishment of a proximity relation, thereby imparting an order to the check the similarity between points\cite{gates2017impact}. Hence, $\Gamma_{p-q}(t) = 1$ implies that $p$ is similar to $q$ (indicating high similarity or proximity) according to the metric $\Gamma$.
\vskip 1mm

In this study, we adopt the random normed spaces approach for clustering using the k-means method, although classical metrics could have been employed for this purpose. When simple clustring or others approaches are represented by deterministic variables, we will utilize the above metric in Theorem \ref{exem1}, which serve as a reference to illustrate our methodology.

\section{Proposed Method}

\subsection{Random Normed K-means }

Although k-means is widely used and effective in many situations, it also has certain weaknesses and limitations, among which we have tried to address is inability to deal with clusters of varying sizes, given that k-means assigns each point to a cluster according to Euclidean distance, this can cause problems when clusters have significantly different sizes, and sensitivity to the scale of the data, as k-means results can vary according to the scale of the data, since it is based on Euclidean distance. Data normalization is often necessary to obtain more accurate results, but what draws attention is that the Euclidean distance remains the essential cause for these two limitations, plus there are other limitations such as dependence on the initial value of $k$ and lacks adaptability to complex cluster shapes.\\

The methodology introduced in our study presents novel techniques for clustering by advocating the utilization of random normed spaces instead of traditional metric spaces. These methodologies offer a fresh perspective on clustering methods applicable across various domains such as data analysis, pattern recognition, and machine learning. In our proposed approach, random normed spaces serve as the foundation for clustering algorithms \cite{esmaeili2020probabilistic}, offering distinct advantages over conventional metrics by introducing randomness into the norm structure \cite{wu2011distance}. This departure from metric spaces allows for greater flexibility and adaptability in capturing complex relationships within datasets, especially in scenarios where traditional metrics may fall short. Moreover, the integration of random normed spaces facilitates the exploration of non-linear and multidimensional data structures, enabling more nuanced insights and improved clustering performance \cite{hsu2015neural}. Our methodological framework emphasizes the importance of understanding the underlying characteristics of data distributions and leveraging random normed spaces to enhance clustering accuracy and robustness. In the subsequent sections, we elucidate the computational procedures and theoretical foundations of our proposed methodologies, highlighting their potential applications and implications across diverse fields of research and practice.\\
After studying the clustering method, the subsequent focus lies in its practical implementation. In this following, we will examine the iterative calculation method and the necessary statistical properties that determine when the iterative process should cease. Let $(\Omega,\Gamma,\mathfrak{H}_{P})$ be random normed space, given a subset $A \subset \Omega$, the average within-cluster sum of random normed $\Gamma_{i,j}$ for elements $i, j \in A$ is given as follows: 
$$\bar{A}=\left\{
\begin{array}
[c]{lll}%
\frac{1}{r(A)}\sum_{i,j \in A; i<j}\Gamma_{i-j} & \ \ \ \ \ \mbox{if} \ r(A) \geq 2,\\
0 & \ \ \ \ \ \mbox{otherwise},
\end{array}
\right.  $$
where $r(A)$ is the number of elements in $A$. For any fixed $l$ with $1 \leq l \leq n$, $C=\left(c_{1},c_{2},...,c_{l}\right)$ is called l-partition of $\Omega$ if each $c_{i}$ is non-empty cluster of $\Omega$, $c_{i}\cap c_{j}= \emptyset$ with $i \neq j$ and $\bigcup_{i=1}^{l}c_{i}=\Omega$.\\
Let $\mathcal{C}_{l}$ be the set of all l-partition of $\Omega$. For each $C=\left(c_{1},c_{2},...,c_{l}\right) \in \mathcal{C}_{l}$, the total of the average within-cluster sum for $C$ is given by
\begin{equation}
\label{OF}
\mathcal{J}(C)=\frac{1}{r(A)}\sum_{i=1}^{l} \sum_{\substack{i,j \in A \\ i<j}}\Gamma_{i-j} \quad \text{for} \quad r(A) \geq 2.
\end{equation}

Then, given a random normed function $\Gamma$ and the number of clusters $l$, the problem is to maximize $\mathcal{J}(C)$ over $C \in \mathcal{C}_{l}$.
More specifically, given the random normed spaces shown in Theorem \ref{exem1}, then for all $t>0$, the objective function will be given as :
\[\mathcal{J}'(C)={argmax}_{C\in C_l}\sum_{i=1}^{l} \sum_{i,j \in A; i<j}\ \frac{t}{t+{\left\|c_i-x_j\right\|}_2}.\]

Before delving into the main algorithm of our proposed clustering method, we have chosen to introduce an alternative technique for determining the optimal centroid positions within the random normed k-means algorithm, as opposed to the classical method and for overcome the problem of $k$.
\vskip 1mm

\subsubsection{Spectral Sampling for RNKM Centroid Estimation}

The algorithm entitled ''Spectral Sampling for RNKM Centroid Estimation'' represents a clustering strategy that incorporates spectral analysis techniques to determine optimal centroid positions within the k-means algorithm \cite{ren2020improved}. The process begins by determining the appropriate number of $k$ centroids \cite{fahim2021k}, using a modified version of the elbow method. Subsequently, a similarity matrix $W$ is computed from the input data $A$, where each element $W_{i,j}$ is evaluated according to the proximity between the data points $x_i$ and the nearest centroid $c_i$ at iteration $t$.\\
\vskip 1mm

After obtaining the similarity matrix, the algorithm proceeds to calculate the normalized Laplacian matrix $L$, which is a transformation of $W$ that highlights the structure of the data in terms of connectivity between points. This step is essential for the rest of the process, as it paves the way for the extraction of the spectral characteristics of the data. The first k-eigenvectors of the Laplacian matrix $L$ are then calculated. These eigenvectors represent the dimensions in which the data are most extensive and are therefore used to form a new matrix $U$. Each row of $U$ is normalized to ensure that distances between points are preserved in the new feature space. Finally, the algorithm applies k-means clustering to the rows of the normalized $U$ matrix to estimate the positions of the $C$ centroids \cite{song2021weighted}. This approach captures the intrinsic structure of the data and produces more coherent and important clusters, particularly when the data feature nested clusters or non-spherical shapes \cite{tang2007enhancing}.
this algorithm is shown below.
\begin{algorithm}[H]
\caption{Spectral Sampling for RNKM Centroid Estimation}
\begin{algorithmic}[1]
\Require Input data $A$, Determine Number of centroids $K$ using Elbow Method Modified
\Ensure Centroids $C$

\State Compute the similarity matrix W from the data A.
\State $W_{ij} = \Gamma_{i-j}$

\State Compute the normalized Laplacian matrix L using W.
\State $L = I - D^{-1/2}WD^{-1/2}$  

\State Compute the first K eigenvectors V of L. 
\State $L = V\Lambda V^T$

\State Form the matrix U by concatenating the first K eigenvectors of V.
\State $U = [v_1, v_2, \ldots, v_K]$

\State Normalize each row of U to obtain the normalized matrix U.

\State Apply k-means clustering on the rows of U to estimate the centroids C.

\State $U_{ij} = \frac{U_{ij}}{\sqrt{\sum_{k} U_{ik}^2}}$

\Return $C$
\end{algorithmic}
\end{algorithm}

\subsubsection{Algorithm for RNKM}
The ''RNKM Algorithm'' represents our proposed method, serving as a clustering technique aimed at optimizing the allocation of data points to clusters through an iterative process. Commencing with a dataset $A$ and a predetermined number of clusters $m$, the algorithm initiates by determining the initial centroids $C$ of the clusters utilizing a spectral sampling algorithm. This approach facilitates more precise initialization and potentially accelerates convergence towards an optimal solution. Subsequently, the algorithm enters an iterative loop where each data point $x_i$ is assigned to the cluster whose centroid $c_i$ is closest, based on a weighted distance measure that considers a parameter $t$. This step is pivotal as it delineates how data points are grouped based on their proximity to the current centroids. Following the assignment of points to clusters, centroids are updated by averaging the data points within each cluster. This iterative process continues until the centroids' positions exhibit minimal change, indicating algorithm convergence. Additionally, the algorithm conducts k-means clustering across a range of $t$ values, exploring diverse distance scales to identify the configuration yielding optimal clustering results. Evaluation metrics, such as the sum of squares of intra-cluster distances or other indicators of clustering quality, inform the determination of the best clustering outcome. In essence, our proposed method enhances upon classic k-means by initializing centroids based on spectral analysis and dynamically adjusting data point assignment according to an adaptable distance parameter. These refinements lead to the formation of more coherent clusters and have the potential to enhance clustering performance, particularly on complex datasets.
The algorithm's implementation is provided below.

\begin{algorithm}[H]
\caption{RNKM Algorithm}
\begin{algorithmic}[1]

\State \textbf{Input:} Data $A$, choose number of clusters $l$ using Elbow Method. 
\State \textbf{Output:} Cluster centers $\{c_1, c_2, \ldots, c_l\}$

\State Initialize cluster centroids C using spectral sampling algorithm
\Function{RNKM Single $t$ }{$A, l, \text{max\_iters}, t$}
    \State \textbf{Repeat until convergence:}
    \For{$i = 1$ to $r(A)$}
      \For{$t > 0$}
        \State Assign $x_i$ to cluster $j$ such that 
        $\arg\max\limits_{l}\frac{t}{t+{\left\|c_i-x_j\right\|}_2}$ 
      \EndFor
      \For{$i = 1$ to $l$}
        \State Update cluster centroid $c_{i+1} = \frac{1}{|C_i|} \sum_{x_j \in C_i} x_j$
      \EndFor
    \EndFor
    \State \textbf{Return} clusters, centroids
\EndFunction

\State Perform RNKM Single $t$ for a range of $t$ values 

\State Return clusters, centroids and $t$ with best score.
  
\end{algorithmic}
\end{algorithm}

\begin{thm}
Let $(\Omega,\Gamma,\mathfrak{H}_{P})$ be a random normed space as defined in Example \ref{exem1}. Then, each iteration of the RNKM algorithm ensures a monotonic increase in the objective function, as defined in Equation (\ref{OF}), thereby guaranteeing convergence of the algorithm.
\end{thm}
\begin{proof}
To simplify the notation, let us use the shorthand $\mathcal{J}(C_{1}^{(t)},...,C_{k}^{(t)})$ for the
RNKM objective, namely,
\begin{equation}
  \mathcal{J}(C_{1}^{(t)},...,C_{k}^{(t)}) \geq \max_{c_{1},...,c_{l} \in \Omega}\sum_{i=1}^{l} \sum_{x \in C_{i}}\ \frac{t}{t+{\left\|c_i-x\right\|}_2}.
	\end{equation}

Let $\left(C_{1}^{(t-1)},...,C_{k}^{(t-1)}\right)$ be the previous partition and $\left(C_{1}^{(t)},...,C_{k}^{(t)}\right)$ be the new partition assigned at iteration $t$. Using the definition of the objective, we clearly have that
\begin{equation}
\label{eqproof1}
 \mathcal{J}(C_{1}^{(t)},...,C_{k}^{(t)}) \geq \sum_{i=1}^{l} \sum_{x \in C_{i}^{(t)} }\ \frac{t}{t+{\left\|c_i^{(t-1)}-x\right\|}_2}.
\end{equation}
In addition, the definition of the new partition $\left(C_{1}^{(t)},...,C_{k}^{(t)}\right)$ implies that it
maximizes the expression 
$$\sum_{i=1}^{l} \sum_{x \in C_{i}}\ \frac{t}{t+{\left\|c_i^{(t-1)}-x\right\|}_2},$$
 over all possible partitions $\left(C_{1}^{(t)},...,C_{k}^{(t)}\right)$. Therefore, 
\begin{equation}
\label{eqproof2}
\sum_{i=1}^{l} \sum_{x \in C_{i}^{t-1}}\ \frac{t}{t+{\left\|c_i^{(t)}-x\right\|}_2} \geq \sum_{i=1}^{l} \sum_{x \in C_{i}^{(t-1)}}\ \frac{t}{t+{\left\|c_i^{(t-1)}-x\right\|}_2}.
\end{equation}
We have that the right-hand side of Equation (\ref{eqproof2}) equals $\mathcal{J}(C_{1}^{(t-1)},...,C_{k}^{(t-1)})$ .Combining this with Equation (\ref{eqproof1}) we obtain that
$$ \mathcal{J}(C_{1}^{(t)},...,C_{k}^{(t)}) \geq \mathcal{J}(C_{1}^{(t-1)},...,C_{k}^{(t-1)}).$$
Which concludes our proof.
\end{proof}

While the preceding theorem tells us that the RNKM objective is monotonically increasing, there is no guarantee on the number of iterations the RNKM algorithm needs to reach convergence. Furthermore, there is no nontrivial upper bound on the gap between the value of the RNKM objective at the algorithm's output and the maximum possible value of that objective function. In fact, RNKM might converge to a point that is not even a local maximum. To improve the results of k-means, it is often recommended to repeat the procedure several times with different randomly chosen initial centroids (e.g., selecting initial centroids as random points from the data).

\subsection{Time Complexity}
Let $n$ be the number of $d$-dimensional vectors, $k$ the number of clusters in the data set and $i$ the number of iterations required for convergence. To calculate and demonstrate the time complexity of the RNKM algorithm, we first need to determine the complexity of each step and then deduce the total time complexity of this algorithm. For the first step, we choose l clusters using the Elbow method, and the determination of l using the Elbow method involves running k-means for different values of l and calculating the sum of squares within the clusters. If k different values are tried, and each k-means iteration requires i iterations, the complexity can be estimated by $O(k.i.n.d)$.
we can confirm that k is a small constant with respect to n, 

for the second step concerns the initialization of the C cluster centroids using the modified spectral sampling algorithm, since the similarity calculated using the $gamma$ function depends on t, considering that t varies from 1 to T so the complexity of the similarity matrix calculation is given by $O(T. n^2. d)$, to compute the normalized Laplacian matrix L using W is given by $O(T.n^3)$, the computation of the first K eigenvectors V of L is given by $O(T.n^2.k)$, forming the matrix U by concatenating the first K eigenvectors of V and normalizing each row of U is given by $O(T.nk)$, applying k-means clustering on the rows of U to estimate the centroids C is estimated by $O(T.i.nk^2)$. 

 the overall complexity of this algorithm is $O(T.n^3)$, for steps 3 to 7 concerning RNkmeans Single t , the constant t has no influence on the complexity , the complexity is estimated by $O(i.n.l.d)$ , for step 8 concerning execution of RNKmeans Single t on a range of t values , let's consider this range from 1 to T , so the complexity is $O(T.i. n.k.d)$ , the last step is to return clusters , centroids , t with best results , its complexity simply is $O(T)$, The overall complexity of the RNKM algorithm is dominated by spectral sampling initialization and repeated RNKM Single t calculations so we can deduce that the total complexity is $O(T.n^3) + O(T.i.n.k.d)$ hence $O(T.n^3 + T.i.n.k.d)$.

The time complexity \cite{b24, b26} of our proposed algorithm is compared with other algorithms in Table 1.\\

\begin{table}[ht]
\centering
\begin{tabular}{|c|c|}
\hline
Algorithm & Complexity \\
\hline
K-Means & $O(nkd)$ \\
K-Means++ & $O(n^2k^2d)$ \\
Kernel Probabilistic K-Means & $O(nkdi)$ \\
Fuzzy C-means & $O(nkdi)$ \\
Random normed K-Means & $O(T.n^3 + T.i.n.k.d)$ \\
\hline
\end{tabular}
\caption{The complexity of K-means, its variants and RNKM.}
\label{tab:complexity}
\end{table}

This variant has a complexity which is a linear combination of the terms $O(T.i.n.k.d)$ and $(T.n^3)$. The $O(T.i.n.k.d)$ term is similar to that of standard K-Means, but multiplied by the number of draws $t$. The term $(T.n^3)$ is due to the calculation of the norm, which is cubic with respect to the dimensionality of the data.
Comparing the complexity of RNKM with the others, we can see that:
\begin{itemize}
	\item It is potentially more expensive than standard K-Means due to the additional term  $(T.n^3)$.
	\item It could be cheaper than K-Means++ if the number of draws is much smaller than $n$ and $k$.
	\item It is comparable to KPKM and FCM if the number of iterations $i$ is similar to the number of draws $t$ , but with an additional cost due to the $(T.n^3)$ term.
\end{itemize}
\vskip 1mm
\section{Material and environment}
For our research, we used an ASUS ROG Strix computer, equipped with an Intel(R) Core(TM) i7-10870H processor with a base frequency of 2.20GHz, expandable to 2.21 GHz, supported by 16 GB RAM and an NVIDIA RTX 2060 graphics card. This hardware configuration was chosen for its ability to handle intensive computing loads, typical of data science modeling and simulation work.
we used the Anaconda environment with Python, without integrating any additional packages. This minimalist approach was adopted to build our model and to eliminate external features that could influence model performance.
In contrast, for the alternative models we studied, additional packages provided by Anaconda were used. These models benefited from the contribution of specialized libraries, and the incorporation of these packages made it possible to exploit performance optimizations and reduce development time.
This methodology enabled us to carry out a comparative analysis between our model, running in a pure Python environment, and other models that take advantage of advanced libraries. The results highlighted the effectiveness of our simplified model in terms of performance.
\vskip 1mm

\section{Result and discussion}
In this study, we investigated the efficacy of our model across various domains by applying it to 21 distinct datasets. These datasets were carefully chosen to encompass a wide spectrum of fields, spanning from medicine to financial economics and can be categorized into two groups, 11 deterministic datasets and 10 random datasets \cite{zhuang2022wasserstein}. The datasets comprise values that vary greatly in magnitude, ranging from significantly small to excessively large. Consequently, normalization was performed to scale these values within the range of 0 to 1. Detailed descriptions of these datasets can be found in Tables 2 and 3: 

\begin{table}[H]
\centering
\begin{tabular}{|c|c|c|c|c|}
\hline
Data set & Instance & Cluster & Dimension & Reference \\
\hline
Iris & 150 & 3 & 4 & \cite{misc_iris_53} \\
Seed & 210 & 3 & 7 & \cite{misc_seeds_236} \\
Glass & 214 & 3 & 10 & \cite{misc_glass_identification_42} \\
Mall & 200 & 5 & 5 & \cite{ashwani2023mall} \\
Cancer & 569 & 2 & 33 & \cite{misc_breast_cancer_wisconsin_(diagnostic)_17} \\
Heart Disease & 303 & 5 & 13 & \cite{heart_disease_45} \\
Wine Quality & 4898 & 3 & 11 & \cite{wine_quality_186} \\
Spambase & 4601 & 2 & 57 & \cite{spambase_94} \\
MAGIC Gamma Telescope & 19020 & 2 & 10 & \cite{magic_gamma_telescope_159} \\
Rice (Cammeo and Osmancik) & 3810 & 2 & 7 & 
Rice (Cammeo and Osmancik)(2019) \\
Phishing Websites & 11055 & 2 & 30 & \cite{phishing_websites_327} \\

\hline
\end{tabular}
\caption{Real datasets used}
\label{tab:datasets}
\end{table}

\begin{table}[H]
\centering
\begin{tabular}{|c|c|c|c|c|}
\hline
Data set & Description & Instance & Cluster & Dimension \\
\hline
Random Data & 0-100 & 10000 & 3 & 20 \\
Integer Data & 0-100 & 100 & 3 & 10 \\
Normal Data & mean=0, standard deviation=1 & 1000 & 3 & 20 \\
Exponential Data & $\lambda$=0.5 & 5000 & 5 & 40 \\
Uniform discrete data & 0-9 & 10000 & 7 & 50 \\
Binomial data & n=10, p=0.5 & 50000 & 7 & 70 \\
Gamma data & shape=1, scale=2 & 100000 & 7 & 100 \\
Lognormal data & mean=0, standard deviation=1 & 70000 & 10 & 90 \\
Poisson data & $\lambda$=2 & 80000 & 10 & 100 \\
Bernoulli data & p=0.3 & 90000 & 7 & 10 \\
\hline
\end{tabular}
\caption{Random data sets of different distributions used}
\label{tab:synthetic_data}
\end{table}
To assess the performance of various models \cite{b24}, we employ the 'Silhouette' \cite{lenssen2024medoid}, 'Davies-Bouldin' \cite{davies}, and 'distortion' \cite{b1} indices. Utilizing multiple indices is imperative since a single index cannot adequately capture all aspects of cluster quality. The 'Silhouette' index evaluates the similarity of an object to its own group relative to other groups \cite{ben2001support}, thereby reflecting the balance between cohesion and separation.
The Silhouette index $S(i)$ for a point or object $i$, is calculated using the formula:
$$S(i) = \frac{b(i) - a(i)}{\max\left\{ b(i) - a(i) \right\}},$$ 
where $a(i)$ is the average distance from object $i$ to other objects in the same group, and $b(i)$ denotes the average distance from object $i$ to objects in the nearest neighboring group (the group to which $i$ does not belong). \\
\vskip 1mm
The Davies-Bouldin index quantifies group compactness and separation. A lower index signifies higher-quality clusters, as indicated by the formula:
\(DB = \frac{1}{m}\sum_{i = 1}^{m}{\max_{j \neq i}\left( \frac{\lambda\left( C_{i} \right) + \lambda\left( C_{j} \right)}{d\left( c_{i},c_{j} \right)} \right)},\)
where, \(d\left( R_{i},R_{j} \right)\) is the Euclidean distance between the centers of clusters $i$ and $j$, $m$ is the number of clusters, $(C_{i})$ is cluster $i$, \(\lambda\left( C_{i} \right)\)  signifies the cluster diameter of $i$. \\
\vskip 1mm
The distortion index evaluates the sum of the squares of the distances of the points from the center of their group\cite{vinh2009information}, calculated by the following formula:

$$D = \sum_{i = 1}^{m}{\sum_{x_{j} \in \widetilde{C_{i}}}\left\| c_{i}- x_{j} \right\|^{2}},$$
where $m$ is number of clusters, $x_{j}$ denotes point $j$, $\widetilde{C_{i}}$ represents the cluster set, and \(c_{i}\) signifies the centroid of $\widetilde{C_{i}}$.
We also used  the Calinski-Harabasz index, also known as the proportional variance criterion, is a metric used to assess clustering quality. It is based on the ratio between the sum of intra-cluster dispersions and the sum of inter-cluster dispersions.\\
Calinski's principal equations are defined as follows.
Let \( \mathcal{C} = \{C_1, C_2, \ldots, C_k\} \) be a set of clusters, where \( k \) is the number of clusters. Let \( n \) be the total number of data points, and \( \bar{x} \) the global centroid of the data.\\
Intra-cluster scatter (SSW):
\[ \text{SSW} = \sum_{i=1}^{k} \sum_{x \in C_i} \| x - \bar{x}_i \|^2,\]
where \( \bar{x}_i \) is the centroid of the cluster \( C_i \).\\
Inter-cluster scatter (SSB):
\[ \text{SSB} = \sum_{i=1}^{k} |C_i| \| \bar{x}_i - \bar{x} \|^2, \]
where \( |C_i| \) is the number of points in the cluster \( C_i \).\\
The Calinski-Harabasz index is then defined by :
\[ \text{CH} = \frac{\text{SSB} / (k - 1)}{\text{SSW} / (n - k)}. \]
A high Calinski-Harabasz index indicates better cluster separation. It is used to compare different clustering results and select the optimal number of clusters. 
Finally, we used The Adjusted Rand Index (ARI) is a metric used to measure the similarity between two partitions of a data set\cite{gates2017impact}, taking into account random permutations.\\
The principal equations of ARI are defined as follows.
Let \( \mathcal{U} = \{U_1, U_2, \ldots, U_r\} \) and \( \mathcal{V} = \{V_1, V_2, \ldots, V_s\} \) be two partitions of a set of \( n \) data points.\\
Contingency matrix:
Let's construct a contingency matrix \( \mathbf{M} \) where each element \( m_{ij} \) represents the number of points in common between the cluster \( U_i \) of the partition \( \mathcal{U} \) and the cluster \( V_j \) of the partition \( \mathcal{V} \).\\
Sum of combinations :
\[ a = \sum_{i} \sum_{j} \binom{m_{ij}}{2}, \]
\[ b = \sum_{i} \binom{|U_i|}{2} - a, \]
\[ c = \sum_{j} \binom{|V_j|}{2} - a, \]
\[ d = \binom{n}{2} - (a + b + c), \]

where \( \binom{n}{2} \) is the number of ways to choose 2 elements from \( n \).\\
Rand index: $\text{RI} = \frac{a + d}{\binom{n}{2}}$.\\

Adjusted Rand Index (ARI):
\[ \text{ARI} = \frac{\text{RI} - \mathbb{E}[\text{RI}]}{\max(\text{RI}) - \mathbb{E}[\text{RI}]}, \]

where \( \mathbb{E}[\text{RI}] \) is the expectation of the Rand index for random partitions.

The ARI varies between -1 and 1. A value of 1 indicates a perfect match between the two partitions, a value of 0 indicates that the similarity between the partitions is what would be expected by chance, negative values indicate a match worse than chance.\\

These two indices, Calinski-Harabasz and ARI, are powerful tools for assessing cluster quality and comparing different clustering methods.\\

The following table presents a summary of the results obtained from our search such for each algorithm on a data, we note  Silhouette Index (SI), Davies (Da), Distortion (Di), Calinski (Ca) and Adjusted Rand Index (ARI).
\begin{table}[H]
\centering
\begin{tabular}{|c|c|c|c|c|c|c|}
\hline
Data set&Index&KM &KM++ &KPKM & FCM &RNKM\\
\hline
Iris&SI&0.459 &0.459&0.544&0.549&\textbf{0.680} \\
&Da&0.833&0.833&0.675&0.669&\textbf{0.405}  \\
&Di& 139.82&139.82&79.40&79.36&\textbf{78.85} \\
&Ca& 452.12&452.12&448.77&452.12&\textbf{513.92} \\
&ARI&0.449&0.449&0.442&0.449&\textbf{0.539} \\

Seed &SI&0.471 &0.400 &0.400 &\textbf{0.671} &0.651 \\
 &Da&0.753 &0.927 &0.919 &\textbf{0.446} &0.456 \\
&Di&0.82 &0.75 &0.75 &\textbf{91.77} &95.274 \\
&Ca&387.02 &372.44 &372.44 &\textbf{500.01} &478.20 \\
&ARI&0.44 &0.43 &0.43 &\textbf{0.519} &0.501 \\

Glass&SI&0.426&0.425&0.131&0.253&\textbf{0.430} \\
&Da&1.293&1.314&1.558&1.622&\textbf{1.274} \\
&Di&1418.66&1418.19&1851.57&1579.43&\textbf{1418.86} \\
&Ca&338.10&337.75&150.24&201.88&\textbf{355.85} \\
&ARI&0.36&0.36&0.28&0.30&\textbf{0.37} \\

Mall&SI&0.467&0.553&0.553&0.553&\textbf{0.554} \\
&Da&0.715&0.572&0.571&0.572&\textbf{0.568} \\
&Di&106348.37&44448.45&44454.47&45067.13&\textbf{44537.0}\\
&Ca&230.97&246.57&246.57&246.57&\textbf{247.11}\\
&ARI&0.0023&0.003&0.003&0.003&\textbf{0.0031}\\

Cancer& SI&0.344&0.344&0.120&0.339&\textbf{0.345} \\
& Da&1.312&1.312&1.905&1.324&\textbf{1.311} \\
& Di&11595.46&11595.46&14400.91&11777.18&\textbf{11595.45} \\
&  Ca&197.114&197.114&79.887&200.03&\textbf{11595.45} \\
& ARI&0.51067&0.51067&0.182&0.511&\textbf{0.52} \\

heart disease& SI& 0.1138&0.1138&0.120&0.22&\textbf{0.25} \\
             &Da& 2.1751&2.1751&1.905&2.001&\textbf{1.982} \\
             &Di&2604.4234&2604.4234&2557.065&2487.022&\textbf{2411.833} \\
             &Ca& 35.2208&35.2208&35.26&36.009&\textbf{36.11} \\
             &ARI& 0.1464&0.14647&0.148&0.151&\textbf{0.164} \\
\hline
\end{tabular}

\caption{Results of different models on all data}
\label{tab:result}
\end{table}

The table presents a concise overview of the performance metrics for various clustering algorithms across multiple datasets \cite{liu2021kernel,izakian2011fuzzy}. It effectively highlights the strengths and weaknesses of each algorithm in terms of clustering accuracy, intra-cluster distance and computational time. We will discuss these results in detail later to gain deeper insights and implications for our research.
Upon analysis of various datasets, the RNKM algorithm exhibited superior performance on the Iris dataset across all evaluated metrics \cite{shutaywi2021silhouette,faisal2020comparative}, underscoring its adeptness at managing datasets characteristic of such attributes. In contrast, when applied to the Seed dataset, the FCM algorithm surpassed its counterparts in two out of three metrics, hinting at its advantageous application to datasets prone to cluster overlap. The examination of the Glass dataset revealed a tightly contested field, where RNKM marginally exceeded the performance of competing algorithms in the initial metric. Similarly, an assessment of the Mall dataset indicated a slight edge for RNKM, albeit with negligible discrepancies among the algorithms' performances. Lastly, the Cancer dataset outcomes were closely contested; however, RNKM managed to secure the leading score in the final metric by a slender margin, suggesting its efficacy in this context.

\begin{table}[H]
\centering
\begin{tabular}{|c|c|c|c|c|c|c|}
\hline
Data set&Index&KM &KM++ &KPKM & FCM &RNKM\\
\hline
Wine Quality& SI&0.23507&0.23508&0.22&0.237&\textbf{0.2411} \\
&Da&1.48403&1.48403&1.4700&1.39&\textbf{1.376} \\
& Di&45577.0063&45577.0063&45562.32&45547.46&\textbf{45498.9} \\
& Ca&1844.4573&1844.4573&1851.7&1858.21&\textbf{1872.54} \\
& ARI&0.01802& 0.01802& 0.0181& 0.01821&\textbf{ 0.019} \\

Spambase&  SI&0.65963& 0.658& 0.63& 0.661&\textbf{ 0.67} \\
& Da&0.61265&0.612&0.613&0.611&\textbf{0.6109} \\
& Di&242529.0659&242529.0659&242611.0659&242476.32&\textbf{242455.81} \\
&  Ca&374.094&374.094&371.02&379.4&\textbf{381.58} \\
& ARI&-0.0049&-0.0049&-0.0051&-0.0041&\textbf{-0.0036} \\

MAGIC  &SI&0.29309&0.29309&0.30&0.27&\textbf{0.33} \\
Gamma& Da&1.4393&1.4393&1.43&1.45&\textbf{1.40} \\
Telescope& Di& 136926.7198&136926.7198&136926.73&136926.91&\textbf{1368223.43} \\
& Ca& 7399.2378&7399.2378&7401.02&7389.42&\textbf{7421.003} \\
& ARI&0.00616&0.00616&0.0062&0.00569&\textbf{0.00668} \\

Rice & SI&0.40996&0.40998&0.41&0.42&\textbf{0.44} \\
(Cammeo & Da&0.9512&0.9512&0.94&0.938&\textbf{0.921} \\
and &Di&13928.78723&13928.78723&13911.05&13905.17&\textbf{13891.2} \\
Osmancik)& Ca&3483.3282&3483.3282&3485.61&3492.08&\textbf{3507.92} \\
&ARI&0.6815&0.6815&0.691&0.696&\textbf{0.70} \\

Phishing & SI&0.27965&0.27965&0.25&0.265&\textbf{0.29} \\
Websites& Da&1.9163&1.9162&1.92&1.811&\textbf{1.762} \\
& Di&282208.7645&282208.7645&282310.5&282241.002&\textbf{282174.25} \\
&Ca&1936.4174&1936.4175&1930.11&1934.86&\textbf{1948.09} \\
&ARI&0.002553&0.002553&0.0024&0.00251&\textbf{0.0027} \\

\hline
\end{tabular}

\caption{Results of different models on all data}
\label{tab:result}
\end{table}
Analysis of the results of the table comparing the performance of different clustering algorithms on several datasets shows that the Random Normed K-means (RNKM) algorithm stands out for its superior performance. The algorithms compared include K-means (KM), K-means++ (KM++), Kernel Probabilistic K-means (KPKM), Fuzzy C-means (FCM) and RNKM. For the Wine Quality dataset, RNKM obtained the best values for all indices: an SI value of 0.2411, a Da value of 1.376, a Di distortion of 45498.9, a Ca value of 1872.54 and an ARI of 0.019, indicating better cluster cohesion, separation and quality. 
Similarly, for the Spambase dataset, RNKM stands out with an SI value of 0.67, a Da value of 0.6109, a Di distortion of 242455.81, a Ca value of 381.58 and an ARI of -0.0036, showing superiority in terms of compactness and cluster separation.
For the MAGIC Gamma Telescope dataset, RNKM also obtains the best values, with an SI of 0.33, a Da of 1.40, a Di distortion of 1368223.43, a Ca of 7421.003 and an ARI of 0.00668, confirming its robustness.The Rice dataset (Cammeo and Osmancik) shows similar results with RNKM obtaining an SI of 0.44, a Da of 0.921, a Di distortion of 13891.2, a Ca of 3507.92 and an ARI of 0.70, indicating better clustering quality. 
Finally, for the Phishing Websites dataset, RNKM continues to outperform the other algorithms with an SI of 0.29, a Da of 1.762, a Di distortion of 282174.25, a Ca of 1948.09 and an ARI of 0.0027. In summary, RNKM shows superior performance on all the datasets tested, achieving the best values for almost all performance indices, suggesting that it is a more accurate and efficient clustering method compared to the other algorithms evaluated.
\begin{table}[H]
\centering
\begin{tabular}{|c|c|c|c|c|c|c|}
\hline
Data set&Index&KM &KM++ &KPKM & FCM &RNKM\\
\hline

Exponential &SI&0.0188&0.0188&  0.0037&  0.0012&\textbf{0.356}\\ 
data &Da&5.272 &5.272&  7.801&  5.957&\textbf{0.505}\\                             &Di&\textbf{188349.92}&188349.92&  194519.64&  199999.99&  199850.51\\

&Ca&77.3425&77.3425&  52.008& 43.701& \textbf{81.65}\\ 
&ARI&-0.0001753&-0.0001753& -0.00019&  -0.000283&\textbf{-0.00014}\\ 
       
Gamma  &SI&  0.0101& 0.0101& -0.0005& -0.0018&\textbf{0.2427}\\
data&Da&7.7868& 7.7868&12.9260& 9.0271&\textbf{0.6310}\\  
&Di& 9686026.51&\textbf{9686026.49}& 9893027.77& 9999999.97& 9999753.61\\
&Ca& 539.7245&539.7245&235.91 &203.07 &\textbf{678.93}\\	
&ARI& 0.000004087&\textbf{0.000004087}& 0.0000002& 0.00000024& 0.000004068\\

Lognormal &SI& 0.02201&0.022&0.003&-0.059&\textbf{0.858}\\
data  &Da&5.624 &  5.624&7.606&7.678&\textbf{0.098}\\
   &Di& 5989152.61& \textbf{5989152.60}& 6099553.48&6299999.97&6291783.59\\
  & Ca&61.0065& 61.0065& 24.51&15.03&\textbf{82.93}\\
   & ARI&0.0000328& \textbf{0.0000328}& 0.0000004&0.000003&0.000086\\
   
Random   &SI&0.065&0.07136&0.062&0.056&\textbf{0.549}\\
data &Da&3.078&3.526&3.118&3.324&\textbf{0.555}\\
 &Di&\textbf{9671.91}&10214.79&9704.99&10999.99&104067.88\\
 &Ca&77.4657&77.90&77.56&77.38&\textbf{164.04}\\
  &ARI&0.0000955&0.00014&0.0000917&0.0000912&\textbf{0.000265}\\
  
Integer  &SI&\textbf{0.099}&  0.093&  0.043& 0.058&  0.0756\\
data &Da&\textbf{2.410} &  2.729   & 3.220 & 2.762& 2.552\\
  &Di&\textbf{817.14} &  896.19 & 879.89& 999.99& 834.40\\
  &Ca&\textbf{12.1421} & 12.1421 & 12.12& 12.03& 12.137\\
    &ARI&\textbf{-0.004559} &  -0.004559 & -0.0046& -0.00502& -0.004562\\

\hline
\end{tabular}

\caption{Results of different models on all data}
\label{tab:result}
\end{table}

\begin{table}[H]
\centering
\begin{tabular}{|c|c|c|c|c|c|c|}
\hline
Data set&Index&KM &KM++ &KPKM & FCM &RNKM\\
\hline
    
Normal  &SI&0.035&\ 0.04& 0.007& 0.030&\textbf{0.041}\\                          data &Da&4.195& 4.768& 6.877& 4.487& \textbf{4.407}\\
 &Di&\textbf{18610.21}& 19176.71& 19440.77& 19999.99&18718.38\\
  
 &Ca&36.6186& 30.5& 26.6186& 30.1&\textbf{53.84}\\
  &ARI&0.0000635& 0.000024& 0.000031& 0.000069&\textbf{0.00084}\\
	        
Uniform  & SI& 0.012& 0.012& -0.121& 0.00928& \textbf{0.0121}\\
   discrete  &Da&5.744&5.744& 2.253&  6.3017& \textbf{5.74}  \\
data   &Di&474729.38&474729.38& 499677.36& 499999.99&\textbf{474218.46}\\
    &Ca& 89.2879&89.2879& 21.08& 56.75&\textbf{93.28}\\                      
     &ARI&-0.0001203&-0.0001203& -0.00080& -0.00011&\textbf{0.00038}\\

Binomial  &SI&\textbf{0.0078} &\textbf{0.0078} & -0.0028  &   0.0067 &0.00729\\
 data & Da&7.1108 &\textbf{7.11016} & 18.4248 &  7.57755 & 7.2545\\
& Di&3383094.18 &\textbf{3383094.1} & 3383094.18 &  3499999.99 &  3386893.30\\
& Ca& 287.4698 &\textbf{287.4698} & 135.84 &  287.447 &  287.45\\
&ARI& -0.0000197 &\textbf{-0.0000197} & -0.000032 &  -0.0000199 &  -0.0000198\\

Poisson  &SI&0.006&0.006& -0.002&0.003&\textbf{0.0032}\\
data & Da&7.272&7.272&18.369&8.427&\textbf{7.177}\\
& Di&\textbf{7733318}.87&\textbf{7733318}.87& 7950286.12& 7999999.97&7850406.49\\
& Ca&306.1273.87&306.1273& 104.2& 282.04&\textbf{311.9}\\
& ARI& -0.000000115.87&-0.000000115& -0.00000049& -0.00000007&\textbf{-0.00000005}\\

Bernoulli  &SI&0.0035&0.0035&-0.0399& 0.0026&\textbf{ 0.0054}\\
 data  &Da&9.755& 9.755&12.584&10.699 &\textbf{ 5.739}\\
 &Di&\textbf{13206629}.48& \textbf{13206629.48}&13499999.99 &13499066.05&13244904.37\\
  &Ca&38.7436&  38.7436&14.006 &34.27&\textbf{42.87}\\
 &ARI&-0.00002586&-0.00002586&-0.00019 &-0.000028&\textbf{0.0000011}\\
   
\hline
\end{tabular}
\caption{Results of different models on all data}
\label{tab:result1}
\end{table}
\vskip 1mm

The performance of the algorithms on other data sets \cite{fard2020deep}, such as Exponential, Gamma, Lognormal, Random, Integer, Normal, Uniform discrete, Binomial, Poisson, and Bernoulli, varied significantly. RNKM frequently emerged as the top performer, particularly in the first metric, which might suggest its robustness in handling different data distributions. However, it is important to note that the performance of clustering algorithms can be highly dependent on the nature of the data and the evaluation metrics used. While RNKM often showed the best results, Overall, the analysis suggests that RNKM is a strong contender in the clustering algorithm space, often providing superior performance. Nevertheless, the other algorithms also have their merits and may be preferable in certain scenarios.\\
\vskip 1mm

The empirical evidence suggests that the RNKM algorithm, which employs function distribution distances in lieu of the traditional Euclidean metric, confers several benefits in clustering tasks across diverse datasets. The utilization of function distribution distances enables RNKM to perform distribution-sensitive clustering \cite{el2024understanding}, enhancing its responsiveness to the intrinsic distributions of data points and thereby potentially yielding more precise clustering outcomes \cite{singh2023probabilistic}. This methodology also improves the handling of non-spherical clusters, as it is not constrained by the spherical bias inherent to Euclidean distance, thus accommodating clusters with irregular shapes and varied densities. Moreover, the robustness of function distribution distances to outliers offers an advantage, as these measures can diminish the impact of anomalous data points that may skew the overall distribution. The flexibility afforded by the selection of an appropriate function distribution distance allows for a more nuanced definition of cluster similarity, tailored to the specific features of the dataset. Additionally, the potential for better convergence is heightened when the chosen distance measure more accurately reflects the inter-point relationships, steering the algorithm toward an optimal clustering solution. Finally, RNKM exhibits adaptability to different data types, with the capacity to be calibrated for continuous, categorical, or mixed-type data, depending on the function distribution distance implemented.

We generated random data without using a determinate distribution, we obtained the following results represented in the following figures:
\vskip 1mm

\begin{center}

    \begin{tabular}{|c|c|c|}
        \hline
        \textbf{Random before } & \textbf{ K-means} 
      &  \textbf{Random K-means} \\
        \hline
        \includegraphics[width=0.3\linewidth]{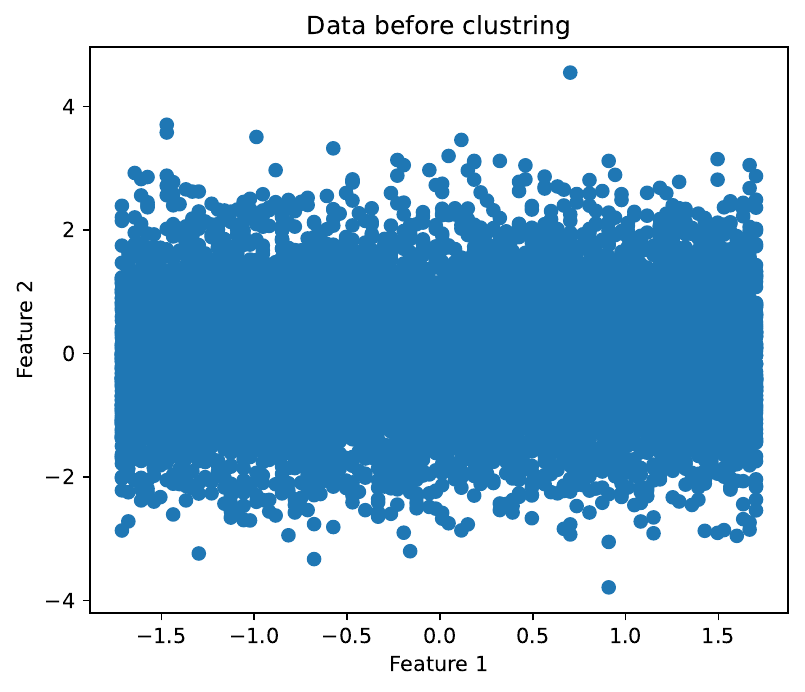} &          \includegraphics[width=0.3\linewidth]{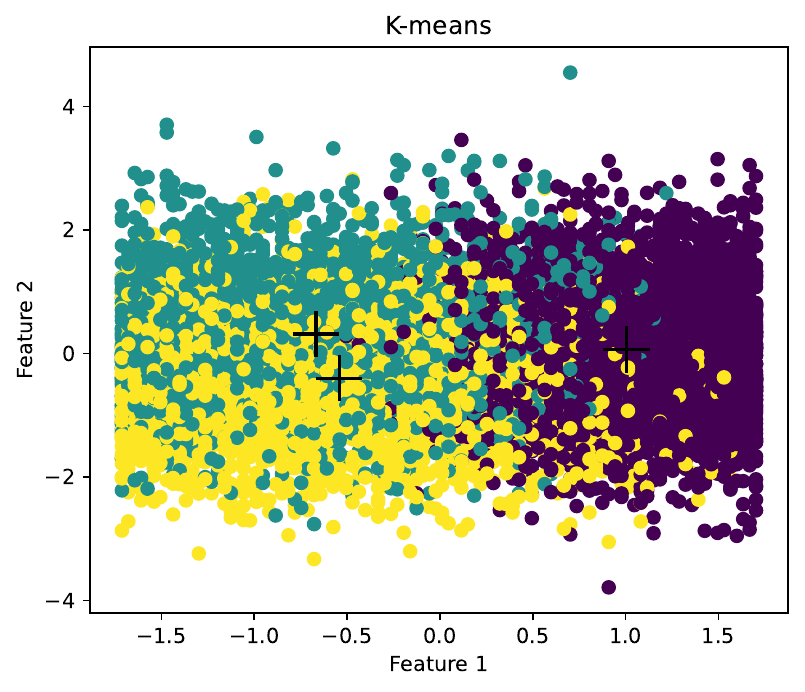} 
  &\includegraphics[width=0.3\linewidth]{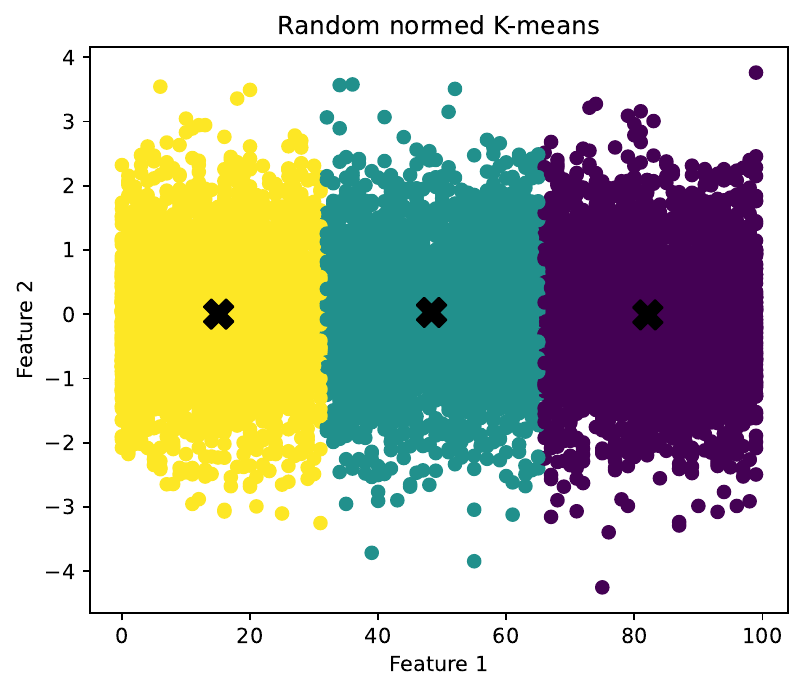}    \\
        \hline
        \textbf{Instance 10000} & \textbf{Number of cluster is 3 }&\textbf{Number of cluster is 3 } \\
        \hline
        Dimension is 20  &  Dimension is 20&  Dimension is 20 \\
        \hline
          Figure 1  & Figure 2& Figure 3 \\
        \hline
    \end{tabular}
  
    \label{tab:}

\end{center}

Figure 1 shows the raw data before any clustering has been applied. The data points are uniformly distributed across the feature space, with no apparent grouping or structure. This scatter plot serves as a baseline, illustrating the initial state of the data.\\

Figure 2 displays the results of applying the K-Means clustering algorithm to the data. The data points are divided into three clusters, each represented by a different color (yellow, teal, and purple). The black crosses indicate the centroids of the clusters. The clusters are arranged horizontally, with the yellow cluster on the left, the teal cluster in the middle, and the purple cluster on the right. The boundaries between the clusters are relatively clear, demonstrating that K-Means has effectively partitioned the data based on its inherent structure.\\

Figure 3 shows the results of applying a variation of K-Means clustering, referred to as "Random Normed K-Means." Here, the data points are also divided into three clusters, represented by yellow, teal, and purple colors. The black crosses again indicate the centroids of the clusters. In this case, the clusters are more distinctly separated, with less overlap between them compared to the standard K-Means clustering. The horizontal arrangement of the clusters is more pronounced, and the clusters appear more compact and well-defined.\\

Overall, The initial scatter plot of the data before clustering reveals no obvious structure or grouping, indicating a uniform distribution. Upon applying K-Means clustering, three distinct clusters with clear boundaries and well-defined centroids emerged, demonstrating the method's capability to effectively partition data based on its inherent structure. A variation, Random Normed K-Means Clustering, produced even more distinct and compact clusters with clearer separations, suggesting an improvement in clustering performance in terms of cluster compactness and separation. Both K-Means and Random Normed K-Means effectively created meaningful clusters from the data. However, Random Normed K-Means seems to offer superior performance for this dataset, yielding more well-defined and distinct clusters.\\

\begin{center}   
    \begin{tabular}{|c|c|}
        \hline
        \textbf{Random gamma data before} & \textbf{Random normed Kmeans} \\
        \hline
        \includegraphics[width=0.4\linewidth]{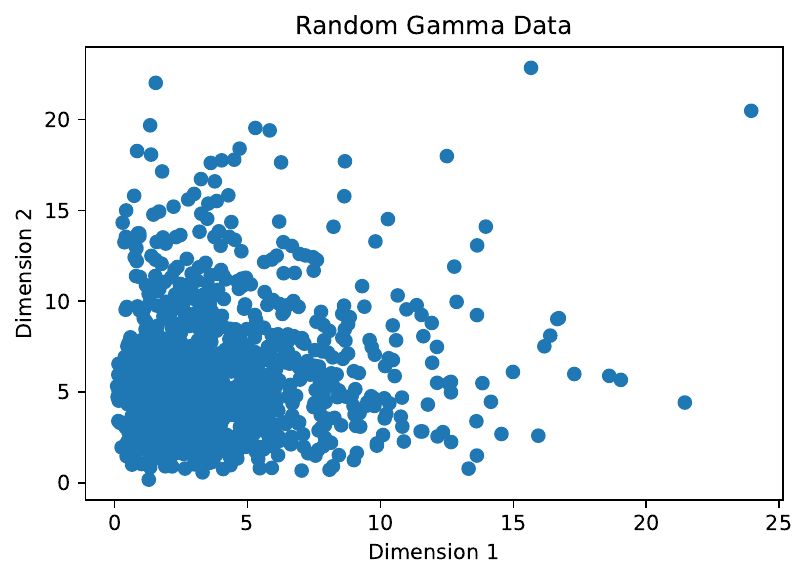} & \includegraphics[width=0.4\linewidth]{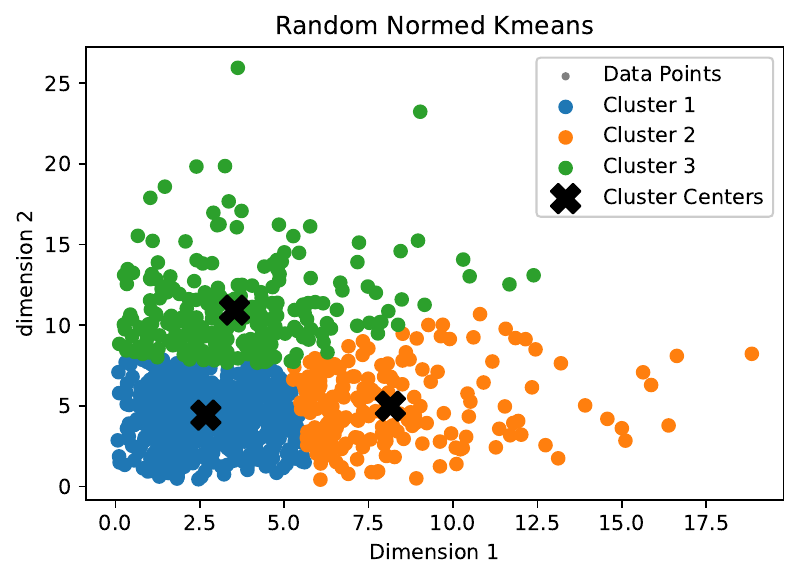} \\
        \hline
        \textbf{Instance 1000} & \textbf{Number of clusters is 3} \\
        \hline
                \textbf{Figure 4} & \textbf{Figure 5} \\
        \hline
    \end{tabular}
   
    \label{tab:1}
\end{center} 	
For the Figure 4 'Data before clustering' , 
This figure shows a scatter plot illustrating a random gamma distribution of data points on the plane, without any form of predefined labeling or grouping. The data points are unevenly distributed, with some areas showing a denser aggregation of points, suggesting potential natural groupings. However, these potential groupings are not explicitly marked or defined, making it difficult to identify clear patterns, trends or anomalies in the dataset. This unstructured nature of the data poses important analytical challenges, as the lack of identifiable groupings makes it difficult to draw meaningful conclusions or make informed decisions on the basis of the data presented.

For Figure 5, after application of the Random Normed K-Means clustering algorithm, this figure reveals a marked transformation in the structure of the data set. The data points are now divided into three distinct clusters, each designated by a unique color (blue, orange and green), providing a clearer understanding of the composition of the dataset. At the center of each cluster is a black “X”, representing the cluster centroid. These centroids indicate the average position of all points within a cluster. 
The clusters exhibit varying densities, with the blue cluster encompassing the densest aggregation of points, followed by the orange and green clusters, each capturing different spatial regions of moderate density. This color-coded segmentation of the dataset into meaningful groups significantly simplifies the data analysis process, enabling the identification of hidden patterns and relationships.
The transition from an unclustered to a clustered dataset vividly illustrates the efficacy of clustering algorithms, such as Random Normed K-Means, in organizing and making sense of data. By grouping similar data points and pinpointing the central tendencies within each cluster, clustering unveils valuable insights and significantly enhances the dataset's analytical utility.\\

The clustering of data using Random Normed Kmeans has far-reaching applications and implications across various fields. It simplifies complex datasets by organizing them into interpretable groups, thereby facilitating easier analysis and insight extraction. Clustering aids in pattern recognition, allowing for the identification of underlying trends that may not be immediately apparent in an unclustered dataset. Additionally, it enables anomaly detection by highlighting points that do not conform to any cluster, which is particularly useful in areas such as fraud detection. In marketing, clustering can be employed to segment customers based on purchasing behaviors, enabling the development of targeted marketing strategies. Furthermore, in image processing, clustering assists in grouping pixels with similar attributes, which is instrumental in object detection and image analysis.

\subsection{The Sensitivity Analysis of RNKM Method}	

When applying the algorithm for data clustering, the parameter $t$ in the distance distribution function plays a crucial role in assessing cluster convergence. This study investigates the impact of $t$ by analyzing evaluation indices in two distinct scenarios, which are using randomly generated data and applying the algorithm to real data. In the case of randomly generated data, we observe significant variations in evaluation indices as a function of $t$. Different values of $t$ can result in divergent results, illustrating the sensitivity of this algorithm to inherent fluctuations in random data. This variability underscores the necessity for careful adjustment of $t$ when utilizing synthetic data. The following Figures 6, 7, 8, 9 and 10 depict the variations of the five evaluation indices corresponding to changes in the parameter $t$ on random gamma data with 300 instances and 3 clusters.\\
\vskip 1mm
\begin{table}[h!] 
\centering
\renewcommand{\arraystretch}{1.5} 
\begin{tabular}{|c|c|c|}
    \hline
    \textbf{Random gamma data} & \textbf{Random gamma data} & \textbf{Random gamma data} \\
    \hline
    \includegraphics[width=0.3\linewidth]{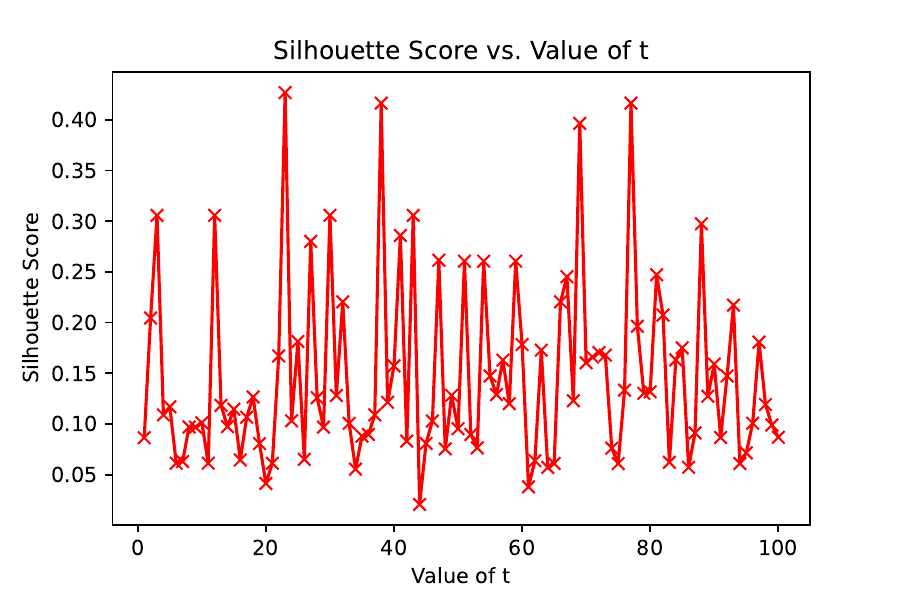} & 
    \includegraphics[width=0.3\linewidth]{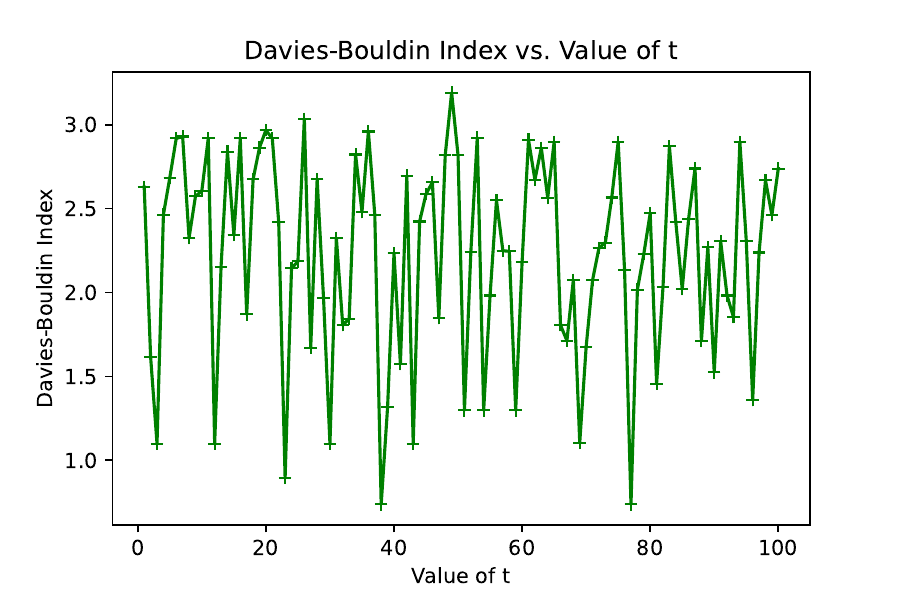} & 
    \includegraphics[width=0.3\linewidth]{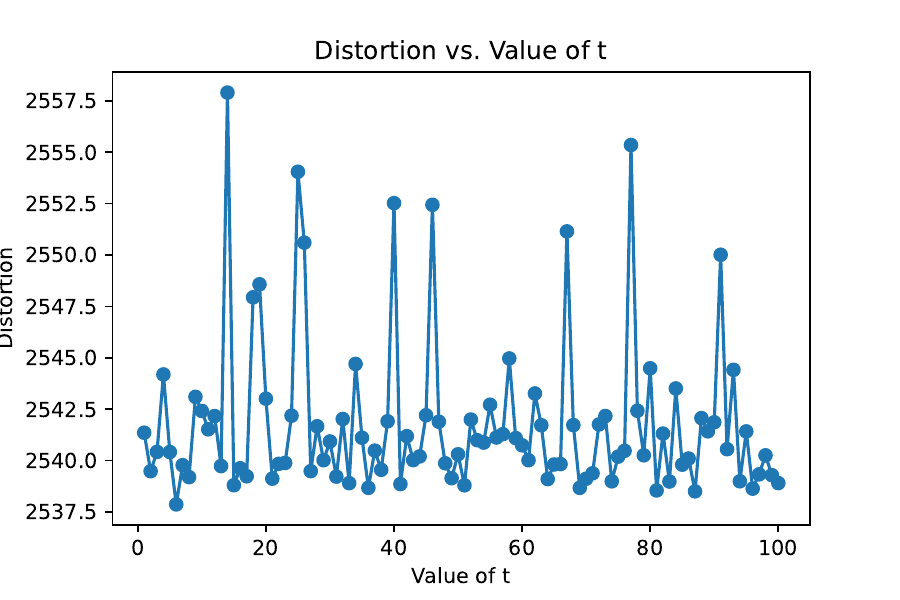} \\
    \hline
    \textbf{Variations in Silhouette index} & 
    \textbf{Variations in Distortion index} & 
    \textbf{Variations in Davies-Bouldin } \\
    \hline
    \textbf{Figure 6} & \textbf{Figure 7} & \textbf{Figure 8} \\
    \hline
    \includegraphics[width=0.28\linewidth]{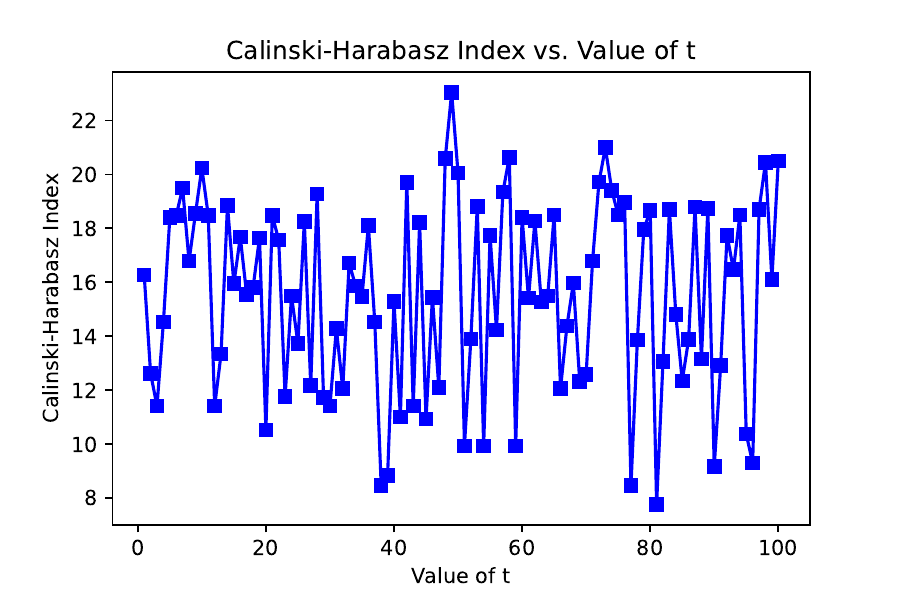} & 
    \multicolumn{2}{c|}{\includegraphics[width=0.3\linewidth]{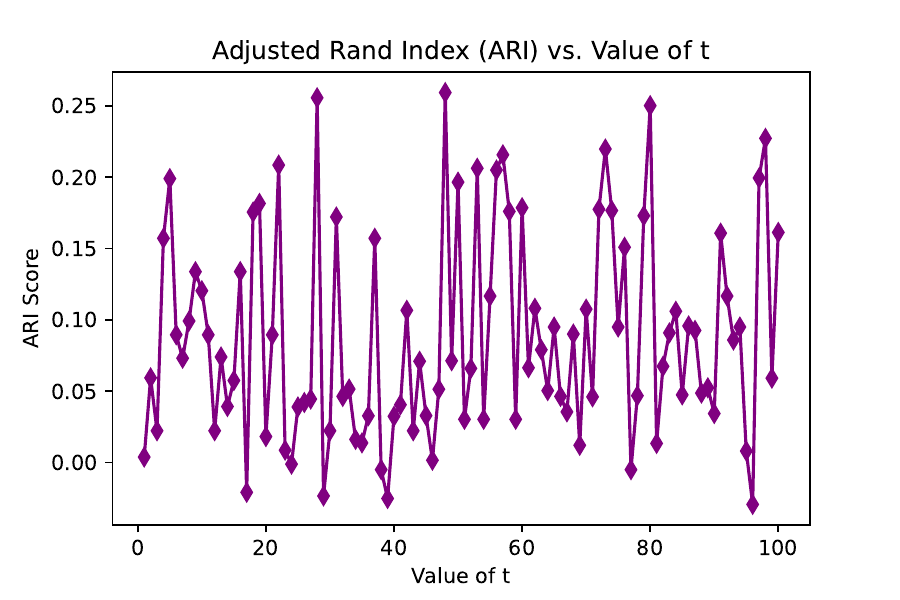}} \\
    \hline
    \textbf{Variations in Calinski-Harabasz index} & 
    \multicolumn{2}{c|}{\textbf{Variations in Adjusted Rand index}} \\
    \hline
    \textbf{Figure 9} & 
    \multicolumn{2}{c|}{\textbf{Figure 10}} \\
    \hline
\end{tabular}
\caption{Description of index variations for random gamma data.}
\label{table:gamma_indices}
\end{table}
\vskip 1mm		
Conversely, when analyzing real data, the evaluation indices demonstrate greater consistency. In this scenario, the results exhibit minimal variation as a function of $t$ and in some cases, even indicate gradual changes. This observation suggests that the underlying structures within real data confer intrinsic robustness to the algorithm, thereby diminishing the dependence on specific choices of $t$. These findings underscore the significance of contextual considerations when configuring the algorithm's parameters. While a cautious approach to selecting $t$ remains crucial for random data, real-world datasets offer a degree of flexibility owing to their inherent stability. This inherent stability can be harnessed in the segmentation process. The subsequent three figures illustrate the variations of the three evaluation indices corresponding to changes in the parameter $t$ on random Iris data.
\vskip 1mm
\begin{table}[h!]
\centering
\renewcommand{\arraystretch}{1.5} 
\begin{tabular}{|c|c|c|}
    \hline
    \textbf{Iris data} & \textbf{Iris data} & \textbf{Iris data} \\
    \hline
    \includegraphics[width=0.3\linewidth]{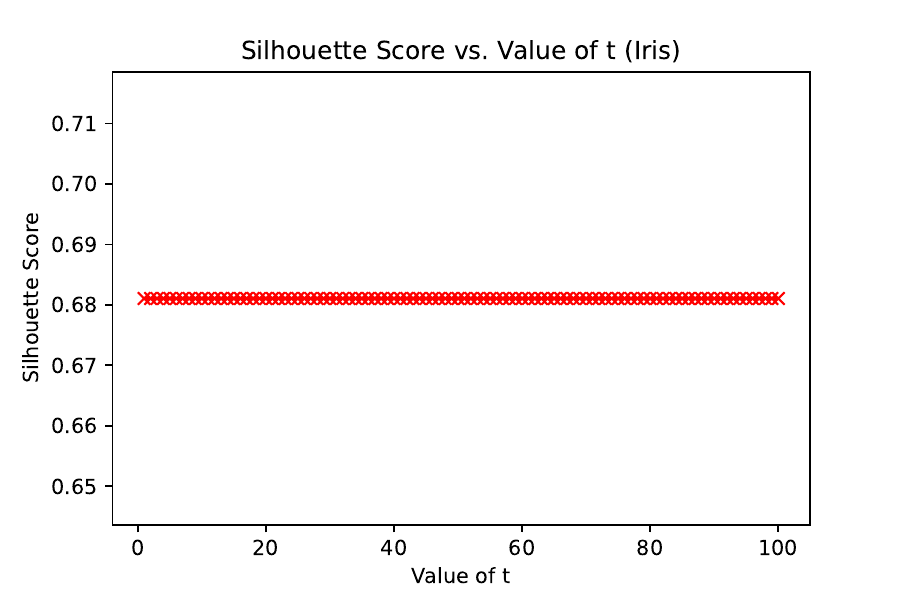} & 
    \includegraphics[width=0.3\linewidth]{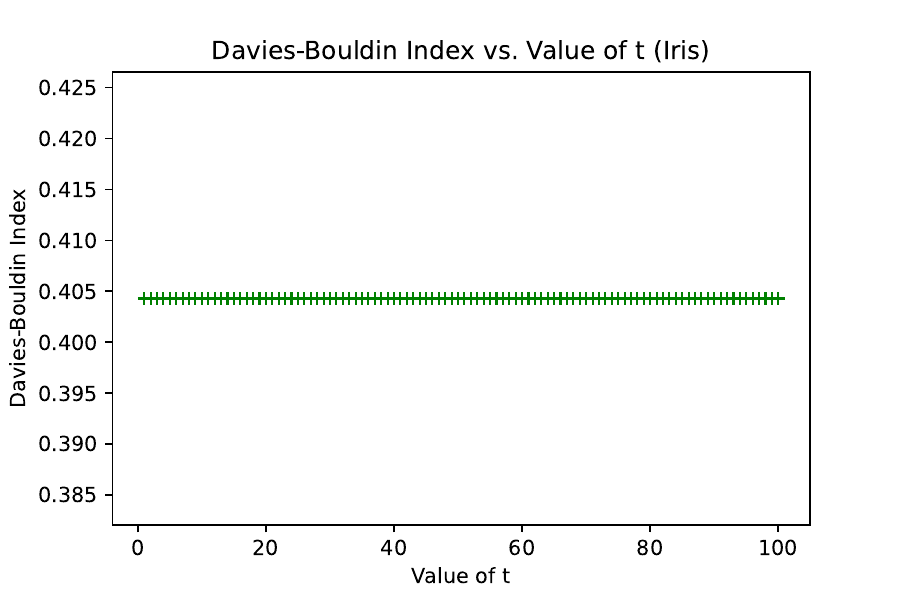} & 
    \includegraphics[width=0.3\linewidth]{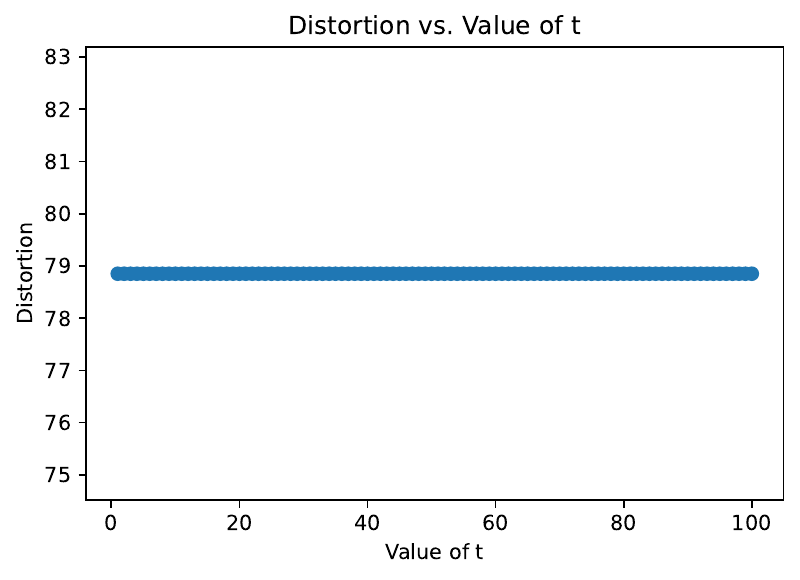} \\
    \hline
    \textbf{Variations in Silhouette index} & 
    \textbf{Variations in Distortion index} & 
    \textbf{Variations in Davies-Bouldin} \\
    \hline
    \textbf{Figure 10} & \textbf{Figure 11} & \textbf{Figure 12} \\
    \hline
    \includegraphics[width=0.3\linewidth]{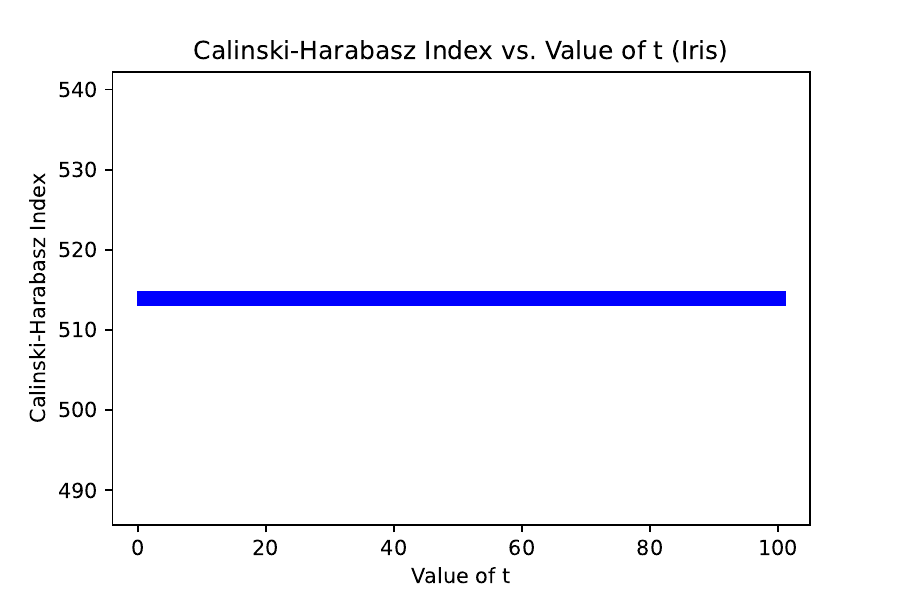} & 
    \multicolumn{2}{c|}{\includegraphics[width=0.3\linewidth]{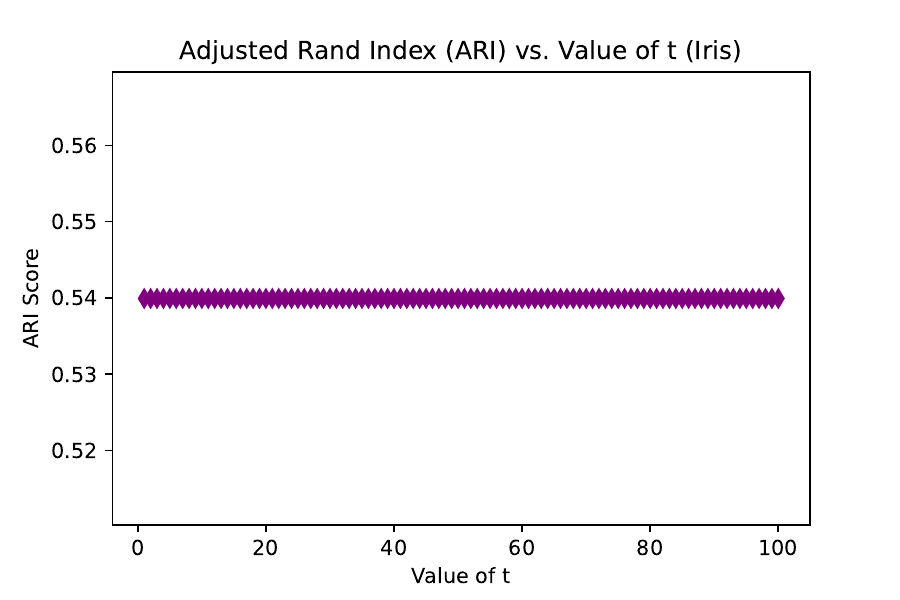}} \\
    \hline
    \textbf{Variations in Calinski-Harabasz index} & 
    \multicolumn{2}{c|}{\textbf{Variations in Adjusted Rand index}} \\
    \hline
    \textbf{Figure 13} & 
    \multicolumn{2}{c|}{\textbf{Figure 14}} \\
    \hline
\end{tabular}
\caption{Description of index variations for Iris data.}
\label{table:iris_indices}
\end{table}

\vskip 1mm
\newpage
\subsection{Comparison with other clustering methods}

In this study, the outcomes derived from our proposed k-means method yielded exceptionally promising results. We conducted a comparative analysis of our approach against existing k-means methods across diverse datasets, encompassing both real-world and synthetic data \cite{sinan2023advances}. Our findings revealed a distinct enhancement in the quality of clusters produced by our method. Notably, there was a substantial reduction in intra-cluster variance, indicating heightened similarity among points within each cluster. Additionally, we observed an augmented inter-cluster separation, signifying improved differentiation among clusters.\\
\vskip 1mm

Upon analyzing evaluation measures such as the silhouette index \cite{lenssen2024medoid}, Davies-Bouldin index \cite{davies} and distortion \cite{b1}, we observed that our method consistently outperformed traditional approaches in the majority of cases. The clusters produced by our method exhibited greater coherence and structural integrity, facilitating their interpretation and subsequent utilization in tasks such as classification or prediction. Additionally, our method demonstrated remarkable accuracy across datasets of varying sizes and random data. By evaluating our model on both real and random datasets, including sets with diverse distributions as depicted in the table above, we validated the robustness and generalizability of our approach for datasets of different characteristics. Our method maintained its performance across small and large datasets, affirming its efficacy in handling substantial volumes of data efficiently. These encouraging findings underscore the potential of our k-means method across various application domains \cite{faisal2020comparative}. Nonetheless, further investigations are warranted to validate and extend these results, particularly by exploring random datasets and those with higher dimensionality, as well as addressing domain-specific challenges. Continued research efforts will contribute to the refinement and broader applicability of our approach.\\

The results obtained provide compelling evidence of the generalizability of our model. Upon application to all seven datasets, we observed a remarkable consistency in performance. The average accuracy, as measured by the metric, demonstrated improvement in the random case and remained satisfactory in the deterministic case. These encouraging outcomes indicate the capacity of our model to effectively adapt to diverse contexts, underscoring its versatility and robustness. Among the real datasets examined, including Iris, Seed, Glass, Mall Customers, and Cancer datasets, our model exhibited commendable accuracy. This suggests its capability to capture the intricate nuances inherent in dataset domains. Notably, even within the random context, where we generated three random exponential datasets, Gamma and lognormal distributions, our model maintained a strong performance compared to existing K-Means variants. The consistency in performance across various datasets prompts critical inquiries regarding the generalizability of our model beyond specific domains. Essential components of the model, such as the distance distribution function with sensitivity $t$, may play a pivotal role in facilitating this adaptability.\\
\vskip 1mm

The RNKM algorithm that we have proposed represents a promising advancement in clustering methodologies, offering distinct advantages over other variants. Specifically, it exhibits greater stability, possesses the ability to accommodate nonlinear data and is more resilient in the presence of outliers. These advantages hold significant value in the realm of real-world data analysis, where datasets often exhibit noise, nonlinear patterns, and outliers. In such scenarios, standard K-Means algorithms may yield suboptimal clustering outcomes. In contrast, the RNKM demonstrates its effectiveness in generating high-quality clustering results, even when applied to random data. This effectiveness stems from the utilization of a probabilistic distance function within the algorithm, which proves to be more adept at capturing data perturbations compared to the conventional Euclidean distance metric. Overall, the RNKM algorithm and its probabilistic distance function offer promising avenues for improved clustering performance across various data types and scenarios.
\vskip 1mm

\subsection{Iteration of $t$ in different data}

The function $\Gamma_{p-q}(t)$ is designed to evolve as the iterations progress, allowing the algorithm to adaptively refine cluster assignments with an increasing focus on the most probable data points. The probabilistic nature of the space is encapsulated by the iteration dependent distance measure, which probabilistically weighs the influence of each point on the centroid's position. As increases, the influence of distant points diminishes, enabling a more precise convergence towards the true cluster centroids. This approach aims to mitigate the sensitivity of K-means to initial conditions and outliers, providing a more robust clustering solution in complex data landscapes. The convergence properties and clustering efficacy of the proposed algorithm are rigorously evaluated against standard benchmarks, demonstrating its potential in extracting meaningful patterns from intricate datasets.\\

In the RNKM algorithm, the parameter $t$ plays an essential role not only in the assignment of data points to clusters, but also in the updating of centroids at each iteration. Indeed, the $t$-weighted assignment gives more importance to proximity as $t$ increases. This has a direct influence on centroid positioning. When $t$ takes on a high value, points are assigned primarily on the basis of their proximity to existing centroids, with little regard for actual distance. This tends to cluster points more compactly around centroids. As a result of this strong aggregation, centroids then tend to be closer to local point densities and better fit the shape of clusters.\\

On the other hand, for small values of $t$, assignment is based more on the actual Euclidean distances between points and centroids. Points can then be assigned to more distant centroids, this results in more spatially dispersed clusters. As a result of this looser distribution of points, centroids tend to be positioned more centrally and globally within clusters, without necessarily coinciding with maximum densities. In this way, we can see that the parameter $t$ allows the compactness of the clusters estimated by Random K-means to vary continuously, and consequently to adjust the position of the centroids between a local or global estimate of the clusters.\\

These images represent the variation of centroid positions according to the first 10 iterartions of $t$ on Mall customers dataset.\\
 
\setcounter{figure}{14} 
\begin{figure}[H]
    \centering
    \includegraphics[width=1\linewidth]{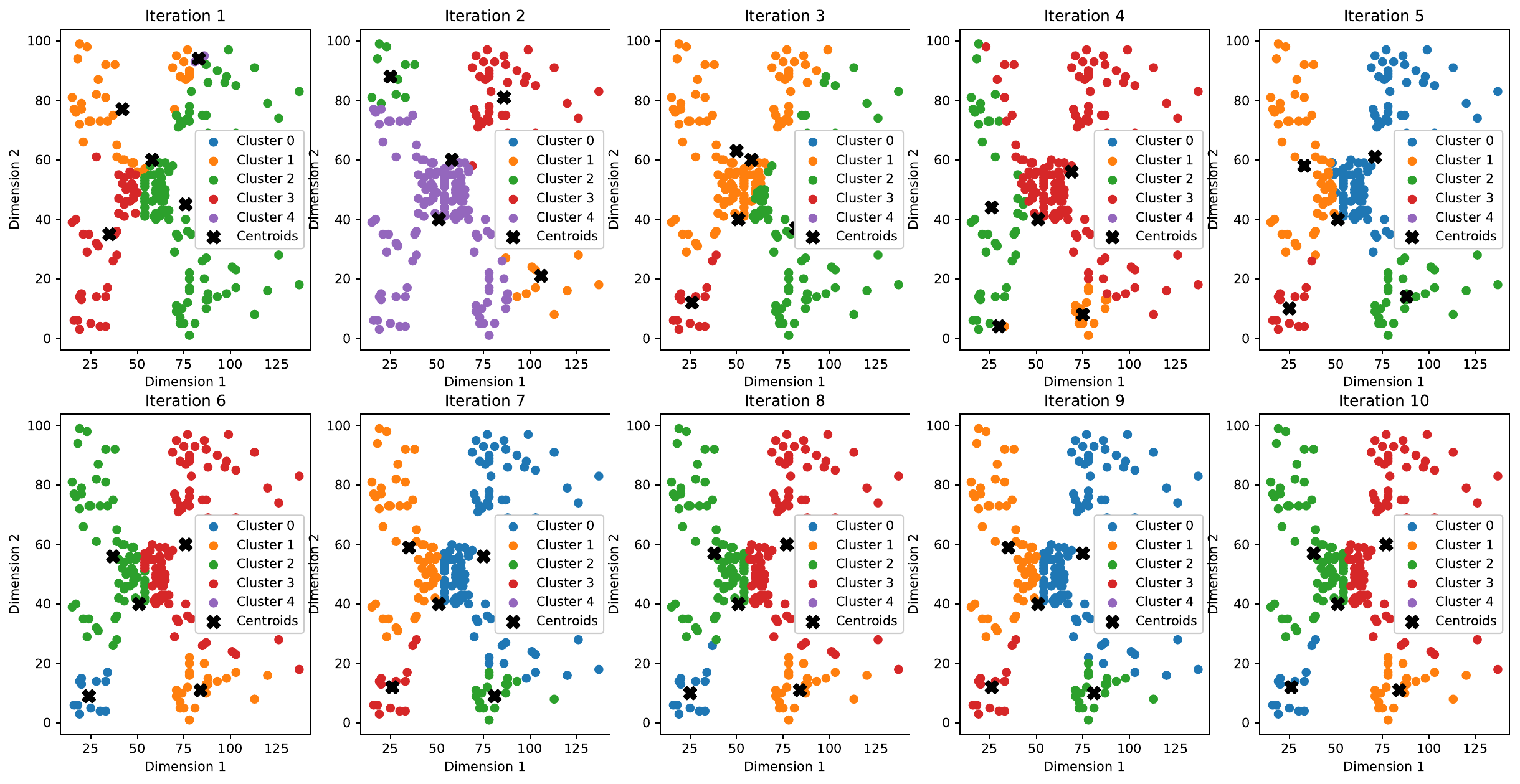}
    \caption{Influence iteration of $t$ on centroid positions in Mall customers dataset}
    \label{fig:2}
\end{figure}
The visualization, consisting of ten iterations, highlights key dynamics of cluster formation and centroid positioning within a two-dimensional space. The data points are colored by cluster assignment, and the centroids, marked by black Xs, evolve over iterations to minimize the weighted distances influenced by t.

First of all, in the early iterations (1 to 2) we have an initialization phase in which centroids change a lot. Cluster (less defined, due to the random influence coming from the starting component). The second one is that there is an alternation phenomena exists for the middle few iterations (3-7). Centroids fluctuated to different positions, vadiable distance function lead them to multiple local optima. We especially see this in the way clusters periodically reorganize. Thirdly, towards the final iterations (8-10), the algorithm tends towards a more stable configuration. Nevertheless, centroid positions continue to adjust slightly, demonstrating the system's ongoing sensitivity to parameter t.

The sensitivity parameter t plays a crucial role in modulating the impact of the Euclidean distance. For higher t values, the influence of the distance diminishes, promoting more uniform clustering, while lower t values make centroids more sensitive to nearby points. In early iterations, significant centroid movements indicate rapid adaptation to local density patterns, leading to visible reorganization of cluster assignments. As iterations progress, the algorithm stabilizes, with centroids converging to equilibrium positions, reflecting an optimal balance of cluster density and separation.

An interesting observation is the alternating patterns in cluster configurations during later iterations, which suggest the algorithm's sensitivity to local optima driven by the interplay between the $\Gamma_{p-q}(t)$ and the underlying data structure. Additionally, the number of clusters appears adaptive, with transitions between four and five clusters across iterations. This behavior, alongside the enhanced sensitivity to local density, underscores the flexibility of the this metric compared to standard Kmeans.

Integration of t into the distance metric enables the algorithm to effectively handle datasets with non-uniform densities and complex cluster shapes. This adaptive behavior, combined with rapid stabilization, offers an advantage over classical K-means, especially in scenarios where traditional approaches might struggle. Future investigations could further analyze the impact of different tt values and compare these results with classical K-means to quantify the improvements achieved.

Similarly, the following images illustrate the fluctuation of centroid positions throughout the initial 10 iterations of $t$ on the Gamma dataset.
\vskip 1mm

\begin{figure}[h]
    \centering
    \includegraphics[width=1\linewidth]{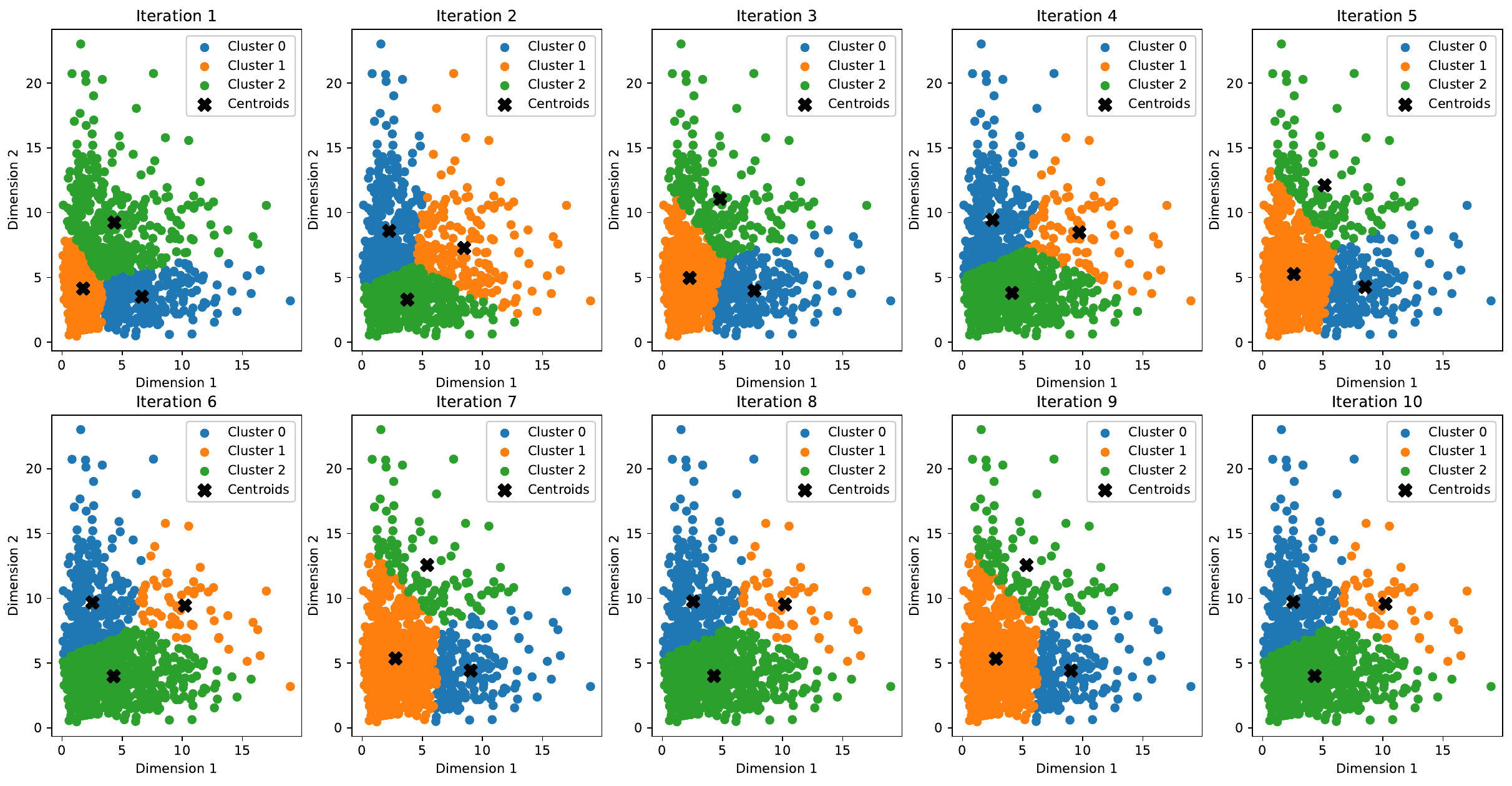}
    \caption{Influence iteration of $t$ on centroid positions in Gamma dataset}
    \label{fig:3}
\end{figure}

This visualization of 10 iterations demonstrates how this adaptation influences the clustering process, particularly the positioning of centroids and the formation of clusters in a two-dimensional space.

In the initial iterations (1-3), the centroids (marked by black X's) exhibit significant movements, rapidly aligning themselves with regions of higher point density. This reflects the algorithm's effort to minimize the normalized distance metric, $\Gamma_{p-q}(t)$, which amplifies the influence of nearby points while reducing the effect of outliers or distant points. Between iterations 4 and 7, the algorithm transitions into a phase of progressive stabilization, where centroid movements become less pronounced, and the clusters start forming clear, well-separated structures.

The oscillatory behavior observed in some centroids suggests the presence of multiple local optima caused by the sensitivity parameter t. This parameter effectively adjusts the balance between local density patterns and global cluster distribution. In later iterations (8-10), the centroids stabilize, and the clusters become compact and well-defined, indicating that the algorithm has reached a near-optimal solution.

$\Gamma_{p-q}(t)$ plays a critical role in shaping the clustering outcome. Compared to standard K-means, it attenuates the influence of distant points, making the centroids more responsive to local density variations. This behavior is particularly beneficial for datasets with irregular densities or outliers, as it prevents extreme points from disproportionately affecting the cluster configuration.

In conclusion, this RNKM approach demonstrates improved adaptability and robustness in handling real-world data with varied densities and distributions. The sensitivity parameter tt enables efficient early optimization and stabilizes cluster formation, reducing sensitivity to outliers. Future work could quantitatively compare this method to classical K-means to further validate the benefits of $\Gamma_{p-q}(t)$.

\vskip 1mm

\section{Conclusion}

The RNKM emerges as a promising solution for data clustering, exhibiting exceptional stability across diverse environments. This advancement heralds new possibilities for its application in real-world scenarios, where algorithmic stability amidst data complexity is paramount for ensuring reliable and meaningful results. Our study's findings underscore the capability of the RNKM in generating high-quality clustering outcomes, particularly in scenarios involving non-linear data, outliers and clusters of unequal sizes. This robustness stems from the utilization of a probabilistic distance function, which effectively captures the data distribution and perturbations, surpassing the limitations of the traditional Euclidean distance utilized in standard K-Means algorithms. The study highlights the random normed K-Means as a potent and versatile clustering algorithm with broad applicability across diverse domains. Nonetheless, further research is warranted to comprehensively assess its effectiveness in real-world applications. Future investigations should focus on enhancing the algorithm's performance by optimizing computation time and result accuracy, as well as extending its applicability to large-scale random datasets. These efforts will aid in refining the RNKM algorithm further and fostering its wider adoption in practical applications.

\bibliographystyle{plain}
\bibliography{bib}
\centering
\end{document}